\documentclass[nohyperref]{article}

\usepackage{microtype}
\usepackage{graphicx}
\usepackage{subfigure}
\usepackage{booktabs} 

\usepackage{hyperref}

\usepackage[accepted]{icml2022}

\usepackage{amsmath}
\usepackage{amssymb}
\usepackage{mathtools}
\usepackage{amsthm}

\usepackage[capitalize,noabbrev]{cleveref}

\theoremstyle{plain}
\newtheorem{theorem}{Theorem}[section]

\newtheorem{lemma}[theorem]{Lemma}

\theoremstyle{definition}

\newtheorem{assumption}[theorem]{Assumption}
\theoremstyle{remark}

\usepackage[textsize=tiny]{todonotes}

\icmltitlerunning{Multirate Training of Neural Networks}

\begin{document}

\twocolumn[
\icmltitle{Multirate Training of Neural Networks} 

\begin{icmlauthorlist}
\icmlauthor{Tiffany Vlaar}{uoe}
\icmlauthor{Benedict Leimkuhler}{uoe}
\end{icmlauthorlist}

\icmlaffiliation{uoe}{Department of Mathematics, University of Edinburgh, Edinburgh, United Kingdom}

\icmlcorrespondingauthor{Tiffany Vlaar}{tiffany.vlaar@mila.quebec}

\icmlkeywords{Neural Networks, ICML}

\vskip 0.3in
]

\printAffiliationsAndNotice{} 

\begin{abstract}
We propose multirate training of neural networks: partitioning neural network parameters into ``fast'' and ``slow'' parts which are trained on different time scales, where slow parts are updated less frequently.
By choosing appropriate partitionings 
we can obtain substantial computational speed-up for transfer learning tasks. We show for applications in vision and NLP that we can fine-tune deep neural networks in almost half the time, without reducing the generalization performance of the resulting models. We analyze the convergence properties of our multirate scheme and draw a comparison with vanilla SGD. We also discuss splitting choices for the neural network parameters which could enhance generalization performance when neural networks are trained from scratch. A multirate approach can be used to learn different features present in the data and as a form of regularization. Our paper unlocks the potential of using multirate techniques for neural network training and provides several starting points for future work in this area. 
\end{abstract}

\section{Introduction}
Multirate techniques have been widely used for efficient simulation of multiscale ordinary differential equations (ODEs) and partial differential equations (PDEs) \cite{Rice1960,Gear1974,GearWells1984, Gunther1993, Engstler1997,Constantinescu2013}. 
Motivations for using multirate techniques are the presence of fast and slow time scales in the system dynamics and to simulate systems which are computationally infeasible to evolve with a single stepsize.  

In their most general formulation the
multirate methods we consider in this work
involve separating the model parameters 
$\Theta$ into multiple components $\Theta_1,...,\Theta_N$ 
corresponding to different time scales. Slow parameters are updated less frequently than their fast counterparts but with larger stepsizes. Synchronization of the parts occurs every slow time step. This is illustrated for two time scales (and accompanying fast $\Theta_F$ and slow $\Theta_S$ parameters) in Figure \ref{fig:MTS}. 
\begin{figure}[h]
    \centering
    \includegraphics[width =\linewidth,trim={0 9cm 0 9cm},clip]{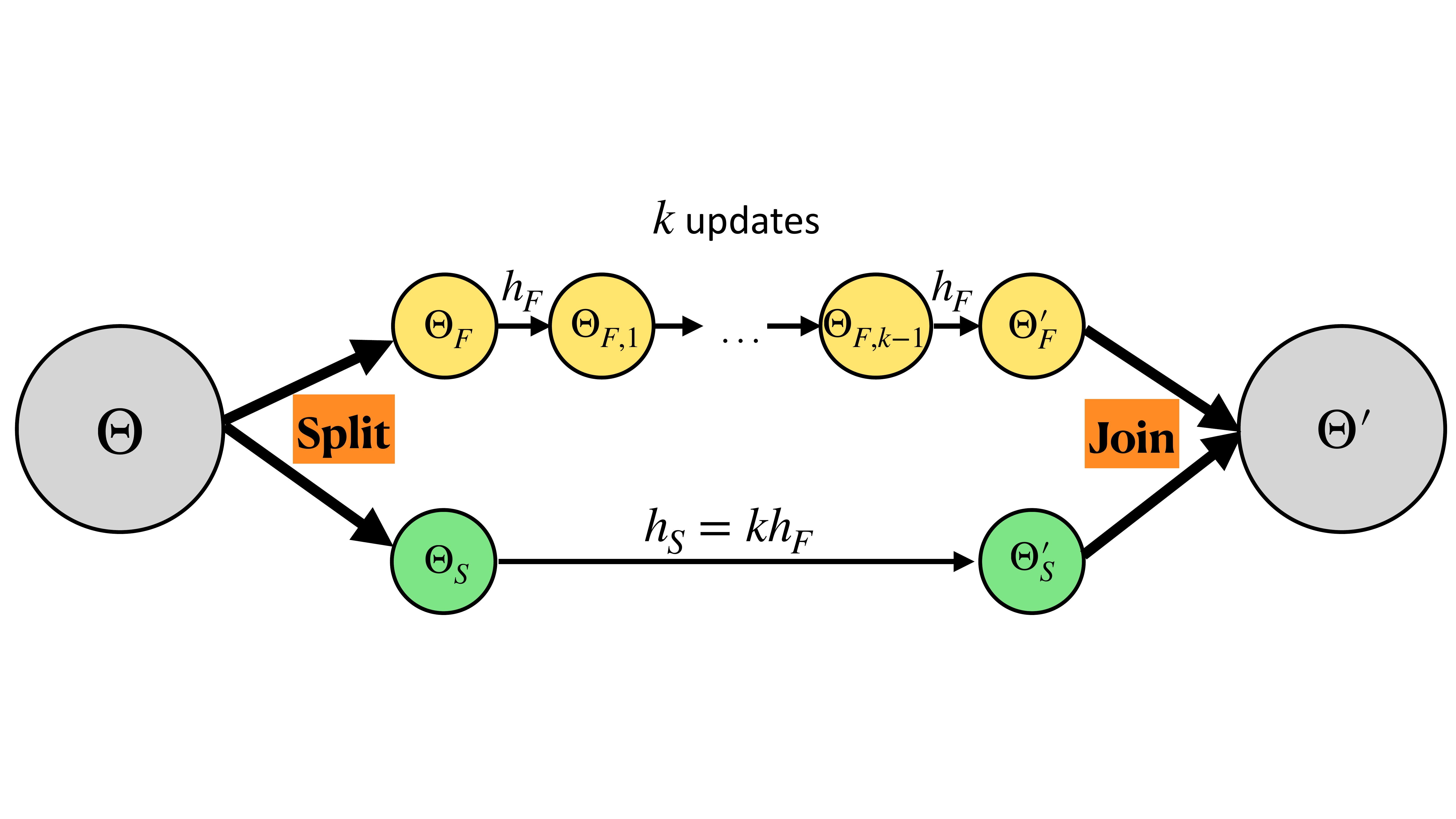}
   \vspace*{-0.4cm}
    \caption{The basic principle of the multirate techniques considered in this paper is illustrated for two time scales in this figure. We first split our model parameters $\Theta$ into fast and slow components, $\Theta_F$ and $\Theta_S$, respectively. The fast components are then updated every step with stepsize $h_F$, whereas the slow components are updated every $k$ steps with stepsize $h_S = k\cdot h_F$.}
    \label{fig:MTS} \vspace*{-0.2cm}
\end{figure}

The idea of using fast and slow weights in a machine learning context has been around for a long time \cite{feldmanfastweights,HintonPlautfastweights,Ba2016}, originally inspired by neuroscience as synapses in the brain have dynamics at different time scales.  However, the use of multirate methods has so far been largely overlooked for this area.
In this work we seek to change this. We propose a novel multirate training scheme and show its use in various neural network training settings.
We describe connections with the current machine learning literature in Section \ref{sec:lit}. 

\begin{figure*}[t]
     \hspace*{-1cm}
    \includegraphics[width=0.98\linewidth]{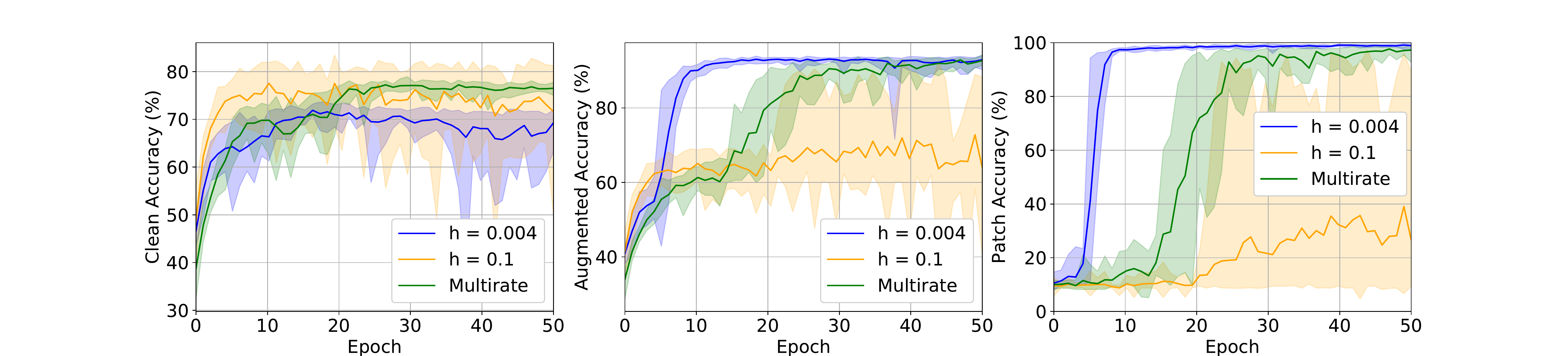} \hspace*{-1cm}
    \includegraphics[width=0.1\linewidth]{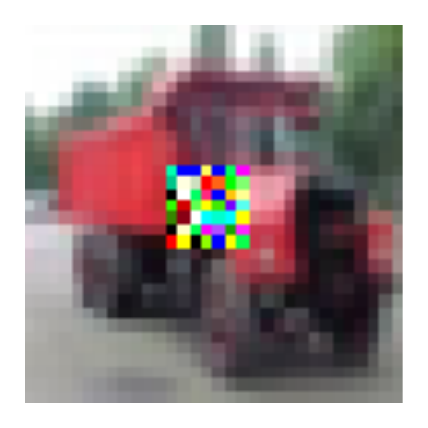}
    \vspace*{-0.3cm}
    \caption{WideResNet-16 architecture trained on 
    patch-augmented CIFAR-10 data \cite{Li2019}. An example of a CIFAR-10 image with a patch is given on the right. Of the training data: 
    20\% is patch-free, 16\% has only the patch, and the rest has both data and patch. More details are provided in Appendix \ref{appx:patch}. Left: clean validation set. Middle: augmented data with patches. Right: patch-only data. A network trained using a small learning rate (blue) learns the patch quickly, whereas a large learning rate (orange) gives higher accuracy on clean 
    data. A multirate scheme (green) trained on both time scales ($h_F = 0.004$, $h_S = 0.1$, see Section \ref{sec:multirate} and Appendix \ref{appx:patch}) is able to memorize the patches and to simultaneously obtain high accuracy on the clean data. 
    \vspace*{-0.2cm}
    }
    \label{fig:patches}
\end{figure*}

To demonstrate how multirate methods may be applicable 
in deep learning applications, consider a WideResNet-16 architecture trained on the patch-augmented CIFAR-10 dataset \cite{Li2019} using SGD with momentum and weight decay and different learning rates (Figure \ref{fig:patches}). In this dataset a noisy patch of $7\times 7$ pixels is added to the center of some CIFAR-10 images. Some images contain both the patch and CIFAR-10 data, while other images only contain the patch or are patch-free. When training using a large learning rate, the network is unable to memorize the patch, but achieves high accuracy on patch-free data. Meanwhile, when training using a small learning rate the network can memorize the patch quickly, but the accuracy on clean data is lower. We demonstrate that a multirate approach
trained on two time scales 
can both memorize the patch and obtain a high accuracy on the patch-free data. 
Multirate methods thus show potential for simultaneously gathering information on different features of the data, for settings where fixed learning rate approaches fail.

In this work we illustrate the benefit of using multirate techniques for a variety of neural network training applications. As main application we use a multirate approach to obtain computational speed-up for transfer learning tasks by evaluating the gradients associated with the computationally expensive (slow) part of the system less frequently (Section \ref{sec:transferlearning}). 
PyTorch code supporting this work, including a ready-to-use torch.optimizer, has been made available at \url{https://github.com/TiffanyVlaar/MultirateTrainingOfNNs}. 

The contributions of this paper are as follows:
\begin{itemize}
    \item We propose multirate training of neural networks, which requires partitioning neural network parameters into fast and slow parts. We illustrate the versatility of this approach by demonstrating the benefits of different partitioning choices for different training applications.
    \item (Section \ref{sec:multirate}) We describe a novel multirate scheme that uses linear drift of the slow parameters during the fast parameter update and show that the use of linear drift enhances performance. We compare its convergence properties to vanilla SGD.
    \item  (Section \ref{sec:transferlearning}) We use our multirate method to train deep neural networks for transfer learning applications in vision and NLP in almost
    half the time, without reducing the 
    generalization performance of the resulting model.
    \item (Section \ref{sec:scratch}) We show that a multirate approach can be used to provide some regularization when training neural networks from scratch. The 
    technique 
    randomly selects new subsets of the neural network to form the slow parameters using an iterative process.
\end{itemize}
We conclude that multirate methods can enhance neural network training and provide a promising direction for future theoretical and experimental work.

\section{Background} 
Multirate methods use different stepsizes for different parts of the system. Faster parts are integrated with smaller stepsizes, while slow components are integrated using larger stepsizes, which are integer multiples of the fast stepsize. Multirate methods have 
been used for more than 60 years \cite{Rice1960} 
in a wide variety of areas \cite{Engstler1997,Gunther1993}. \citet{Gear1974} analyzed the 
accuracy and stability of Euler-based multirate methods applied to a system of ODEs with slow and fast components.

The system of ODEs that forms the starting point for most neural network training schemes is $\text{d}\theta = G(\theta)\text{d}t$, where $\theta\in \mathbb{R}^n$ are the neural network parameters and $G$ represents the negative gradient of the loss of the entire dataset. As a starting point for our multirate approach we partition the parameters as $\theta = (\theta_F,\theta_S)$, with $\theta_F
\in \mathbb{R}^{n_F}, \theta_S
\in \mathbb{R}^{n_S}$, $n = n_F +n_S$,
and obtain system of ODEs:
\begin{align}
    \text{d}\theta_F = G_F(\theta)\text{d}t, \ \ \text{d}\theta_S = G_S(\theta)\text{d}t, \label{GD}
\end{align}
where $G_F$ and $G_S$ are the gradients 
with respect to $\theta_F$ and $\theta_S$, respectively. 

For neural network training the loss gradient is typically evaluated on a randomly selected subset of the training data
and the pure gradient in Eq. \eqref{GD} is subsequently replaced by a noisy gradient which we denote $\tilde{G}(\theta)$. 
Further, most training procedures incorporate momentum \cite{GDwithmom,sutskever}. 
In the stochastic gradient Langevin dynamics method of \citet{WeT11}, the system is further driven by constant variance additive noise.
As a somewhat general model, one may consider a partitioned underdamped Langevin dynamics system of stochastic differential equations of the form: 
\begin{align}
  \text{d}\theta_{\alpha} &= p_{\alpha}\ \text{d}t, \ \ \ \  \text{where} \ {\alpha} = F,S \nonumber \\ 
  \text{d}p_{\alpha} &= \tilde{G}_{\alpha}(\theta)\text{d}t-\gamma_{\alpha} p_{\alpha}\  \text{d}t+\sqrt{2\gamma_{\alpha}\tau_{\alpha}}\ \text{d}W_{\alpha}, 
    \label{Langevin}
\end{align}
with momentum $p= (p_F,p_S)\in \mathbb{R}^n$ and hyperparameters $\gamma_{\alpha},\tau_{\alpha} > 0$.
When evaluating the gradient on the full dataset, Langevin dynamics is provably ergodic \cite{MaStHi2002}, under mild assumptions, and samples from a known distribution. In this paper we will focus on the case $\tau_{\alpha} = 0$, which corresponds to standard stochastic gradient descent (SGD) with momentum under re-scaling of the hyperparameters, however, our multirate approach can easily be extended to the more general case. 
We have also opted to use the same momentum hyperparameter ($\gamma_{\alpha}$ in Eq. \eqref{Langevin}) for both subsystems to provide a fair comparison with standard SGD with momentum. Using different optimizer hyperparameters, as well as exploration of methods which combine different optimizers for different components, is left for future study  (see Section \ref{sec:discussion} and Appendix \ref{sec:variants}).  Algorithms can easily be designed based on partitioning into multiple independent components (not just two) evolving at different rates, as we illustrate in Section \ref{subsec:multirate}.

\section{Multirate Training of Neural Networks}\label{sec:multirate}
In Section \ref{subsec:multirate} we propose a novel multirate technique that can be directly applied to the training of neural networks and discuss application-specific appropriate choices for the fast and slow parameters. In Section \ref{subsec:convergence} we study the convergence properties of the scheme.

\subsection{A Partition-based Multirate Approach}\label{subsec:multirate}  
The type of multirate algorithms we consider in this work
take the following approach for two time scales:
\begin{enumerate} 
    \item Separate model parameters into a fast and slow part.
    \item At every step, compute the gradients with respect to the fast variables. Update the fast variables using the optimizer of your choice with fast stepsize $h_F$.
    \item Every $k \in \mathbb{Z}_{+}$ steps: Compute gradients with respect to the slow variables. Update slow variables using the optimizer of your choice with slow stepsize $h_S = k h_F$.
\end{enumerate}
This multirate approach can be combined with different optimization schemes, such as of the form in Eq. \eqref{Langevin}. In this work, for our analysis and numerical experiments we shall focus on using as base algorithm 
stochastic gradient descent (SGD), where the gradients are computed for every mini-batch of $m$ training examples. We will compare our multirate approach with PyTorch's standard SGD with momentum implementation \cite{Pytorch} and hence for consistency we present our method in the same notation and manner as used in the PyTorch code. Our multirate scheme is described by Algorithm \ref{multirateSGDwithmom}. We refer to the model parameters and momenta associated with the slow system as $\theta_S$ and $p_S$, respectively, and for the fast system as $\theta_F$ and $p_F$. We denote by $\mathcal{L}(\theta_S,\theta_F)$ the neural network loss as evaluated on a minibatch of training examples. We use the cross-entropy loss for classification tasks. We use $\mu$ to denote the momentum hyperparameter, which we typically set to $\mu = 0.9$. 
 
 We discuss variations of Algorithm \ref{multirateSGDwithmom} such as combining this multirate approach with other optimizers, the use of weight decay, or using different initializations for the fast and slow systems in Appendix \ref{sec:variants}. 
 
\textbf{Linear drift.} 
In Algorithm \ref{multirateSGDwithmom} we continuously push the slow parameters along a linear path defined by their corresponding momenta. This means that although the gradients for the slow parameters are only computed every $k$ steps, the slow neural network parameters do get updated every step in the direction of the previous gradient. This is a novel technique for multirate training,
where approaches similar to that in Algorithm \ref{multirateSGDwithmomwithoutlinear} are more prevalent. 
We compare these approaches in 
ablation studies in Section \ref{sec:ablation} and show that the use of linear drift enhances performance.

\begin{algorithm}[H] 
 \caption{Multirate SGD with linear drift} 
 \label{multirateSGDwithmom}
\begin{algorithmic}
\STATE $p_S := \mu p_S + \nabla_{\theta_S} \mathcal{L}(\theta_S,\theta_F) $
 \FOR{$i=1,2,...,k$}
\STATE $p_F := \mu p_F + \nabla_{\theta_F} \mathcal{L}(\theta_S,\theta_F) $ 
\STATE $\theta_F := \theta_F-\frac{h}{k}p_F$ 
\STATE $\theta_S := \theta_S-\frac{h}{k}p_S$ 
 \ENDFOR
 \end{algorithmic}
\end{algorithm}
\begin{algorithm}[H] 
 \caption{Multirate SGD no linear drift} \label{multirateSGDwithmomwithoutlinear}
\begin{algorithmic}
\STATE $p_S := \mu p_S + \nabla_{\theta_S} \mathcal{L}(\theta_S,\theta_F) $
\STATE $\theta_S := \theta_S-hp_S$ 
 \FOR{$i=1,2,...,k$}
 \STATE $p_F := \mu p_F + \nabla_{\theta_F} \mathcal{L}(\theta_S,\theta_F) $ 
\STATE $\theta_F := \theta_F-\frac{h}{k}p_F$ 
 \ENDFOR
\end{algorithmic}
\end{algorithm}

\textbf{Choice of Partitioning.}
Examples of possible separations of the model parameters into fast and slow components are layer-wise, weights vs. biases, or by selecting (random) subgroups. The appropriate separation is application-specific and will be discussed in more detail in upcoming sections.
In Section \ref{sec:transferlearning} we explore obtaining computational speed-up using Algorithm \ref{multirateSGDwithmom} through layer-wise partitioning, where our fast parameters are chosen such that the gradients corresponding to the fast system are quick to compute, while gradients of the full net are only computed every $k$ steps. In Appendix \ref{sec:slowbias} we study the effect of putting the biases of a neural network on the slow time scale. Finally, in Section \ref{sec:randomsubgroups} we use partitioning using random subgroups to develop a regularization technique for neural network training.

\textbf{Extension to more scales.} Although we have presented the algorithm for two time scales, the scheme can easily be extended to more scales. 
To extend our framework to multiple components operating at
$r$ scales, one can use stepsizes $h_{i}=h_{i-1}/K_{i-1}$, $i=1,2,\ldots,r, K_i\in \mathbb{Z}_{+}$, recursively dividing the step sequences in Algorithm \ref{multirateSGDwithmom} and \ref{multirateSGDwithmomwithoutlinear} into finer ones at each successive level of the parameter hierarchy. 

\textbf{Uncoupled learning rates.} In Algorithm \ref{multirateSGDwithmom} and \ref{multirateSGDwithmomwithoutlinear} the fast and slow learning rates are coupled. 
Alternatively, one could introduce an uncoupled learning rate for the slow parameters. This may lead to further performance enhancements, but introduces an extra hyperparameter and thus additional tuning. We provide some ablation studies in Appendix \ref{appx:ablation}. 

\begin{figure*}[h]
    \centering
    \includegraphics[width=0.9\linewidth]{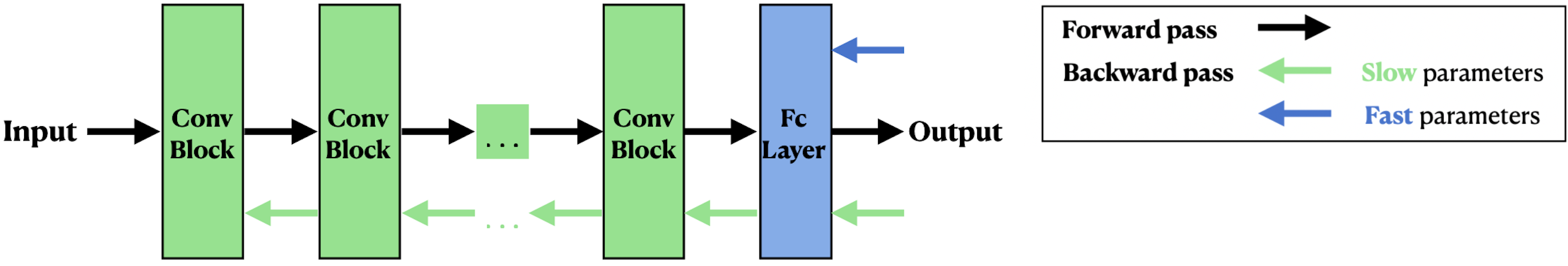} 
    \caption{We indicate in blue the fast parameters
    and in green the slow parameters 
    of a convolutional architecture, which consists of several convolutional blocks (conv block) and fully connected (fc) layer(s). When setting the fast parameters to be the final fc layer(s) (and optionally the conv block directly preceding it), the gradient computation for the backpropagation algorithm is very fast.
   }    \label{fig:resnetfastslow}
\end{figure*}

\subsection{Convergence Analysis}\label{subsec:convergence}
To study the convergence properties of multirate SGD in the non-convex setting we make the following (standard) assumptions:

\begin{assumption} We assume the function $f: \mathbb{R}^n\rightarrow \mathbb{R}$ to be $L$-smooth, i.e., $f$ is continuously differentiable and its gradient is Lipschitz continuous with Lipschitz constant $L > 0$:
\begin{align}
    \|\nabla f(\varphi)-\nabla f(\theta)\|_2\leq L\|\varphi-\theta\|_2, \ \forall \theta,\varphi\in\mathbb{R}^n. 
\end{align}\label{assm:Lsmooth}
\end{assumption} 
\begin{assumption}
\vspace*{-0.2cm}
We assume that the second moment of the stochastic gradient is bounded above, i.e., there exists a constant $M$ for any sample $x_i$ such that
\begin{align}
    \|\nabla f_{x_i}(\theta)\|^2_2 \leq M, \   \ \forall \theta \in \mathbb{R}^n.\vspace*{-0.2cm}
\end{align}\label{assm:boundedvariance}
\end{assumption}
\vspace*{-0.45cm}
Assumption \ref{assm:boundedvariance} guarantees that the variance of the stochastic gradient is bounded. Under Assumption \ref{assm:Lsmooth} and \ref{assm:boundedvariance} we show in Appendix \ref{appx:conv} that Theorem \ref{thm:conv} holds  for our layer-wise partitioned multirate SGD approach: 
\begin{theorem}
\label{thm:conv} We assume that \ref{assm:Lsmooth} and \ref{assm:boundedvariance} hold. Then
\begin{align}
     \frac{1}{T}\sum^{T-1}_{t=0} \mathbb{E}\left [ \|\nabla f(\theta^t)\|^2_2 \right ] &\leq \frac{ 2(f(\theta^0)-f(\theta^*))}{hT} \nonumber \\
     & \ \ \ \ \ + h L M\ell \left ( \frac{1}{3} h L k^2 + 1 \right ),
\end{align}
where $T$ is the number of iterations, $L$ and $M$ are as defined in Assumptions \ref{assm:Lsmooth} and \ref{assm:boundedvariance}, $\ell$ is the number of parameter groups, $k$ is the additional hyperparameter associated with our multirate method, and $\theta^*$ is the optimal solution to $f(\theta)$.
\end{theorem}
From Theorem \ref{thm:conv} one sees that as $T \rightarrow \infty$, the $h L M\ell \left ( \frac{1}{3} h L k^2 + 1 \right )$ term controls the upper bound. 
The expression in Theorem \ref{thm:conv} is very similar to that obtained for vanilla SGD where the rightmost term is replaced by $hLM/2$ (see Appendix \ref{appx:conv}).
Therefore, by decreasing the stepsize $h$, SGD can get closer to the neighborhood of a critical point. For our algorithm the choice of $k$ (the additional hyperparameter introduced by our multirate method) also plays a role, where smaller values of $k$ will lower the upper bound, but also increase the computational cost (in particular for our transfer learning application described in Section \ref{sec:transferlearning}). 

\begin{figure*}[t]
\centering
    \includegraphics[width=0.83\linewidth]{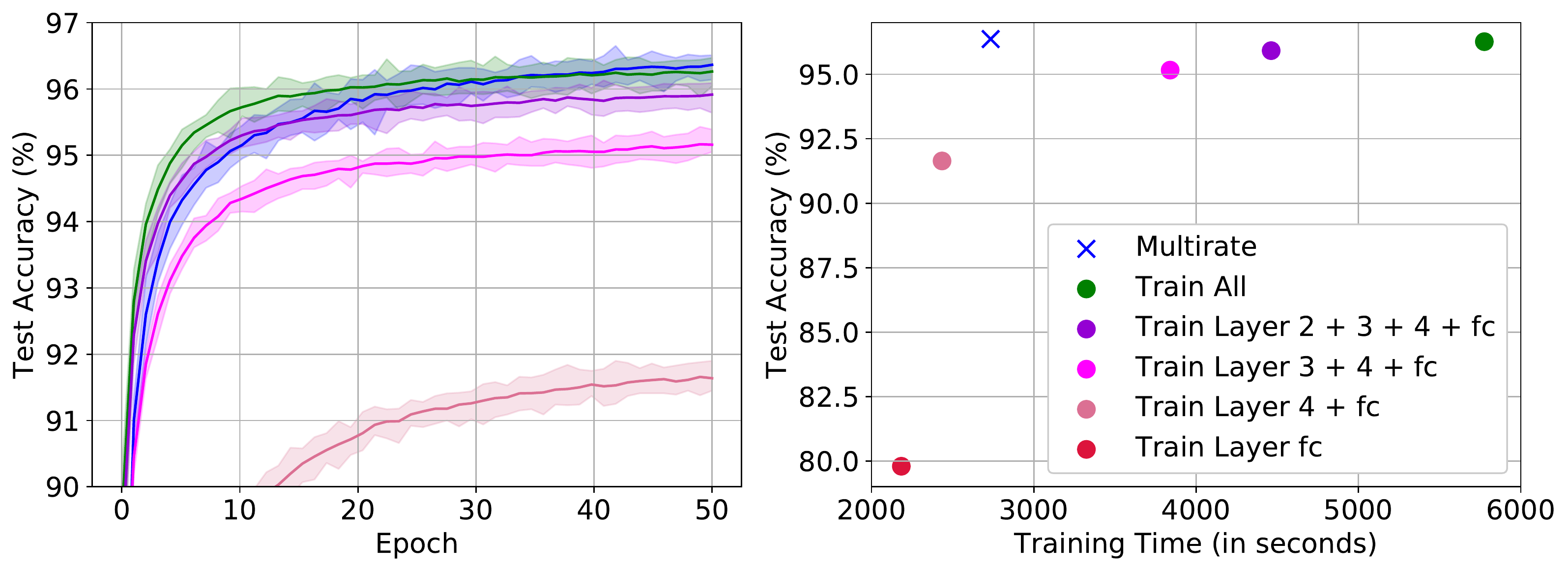}
   \vspace*{-0.35cm}
    \caption{A pre-trained ResNet-34 being trained on CIFAR-10 data using different fine-tuning approaches and our multirate approach (blue).  Results are averaged over 20 runs and all approaches are trained using SGD with momentum as base algorithm. We set $h/k = 0.001, k = 5,$ and $\mu = 0.9$ in Algorithm \ref{multirateSGDwithmom}. The highest test accuracy is reached using our multirate approach (blue), which can be used to train the net in almost half the time. Typical fine-tuning approaches only train the bottom layers of the network, e.g. just the fully connected (fc) layer (red) or layer 4 + fc, which results in a comparable speed-up, but much lower test accuracy.} 
    \label{ResNet-34_rescaled}
\end{figure*}

\begin{figure*}[t]
\centering
    \includegraphics[width=0.83\linewidth]{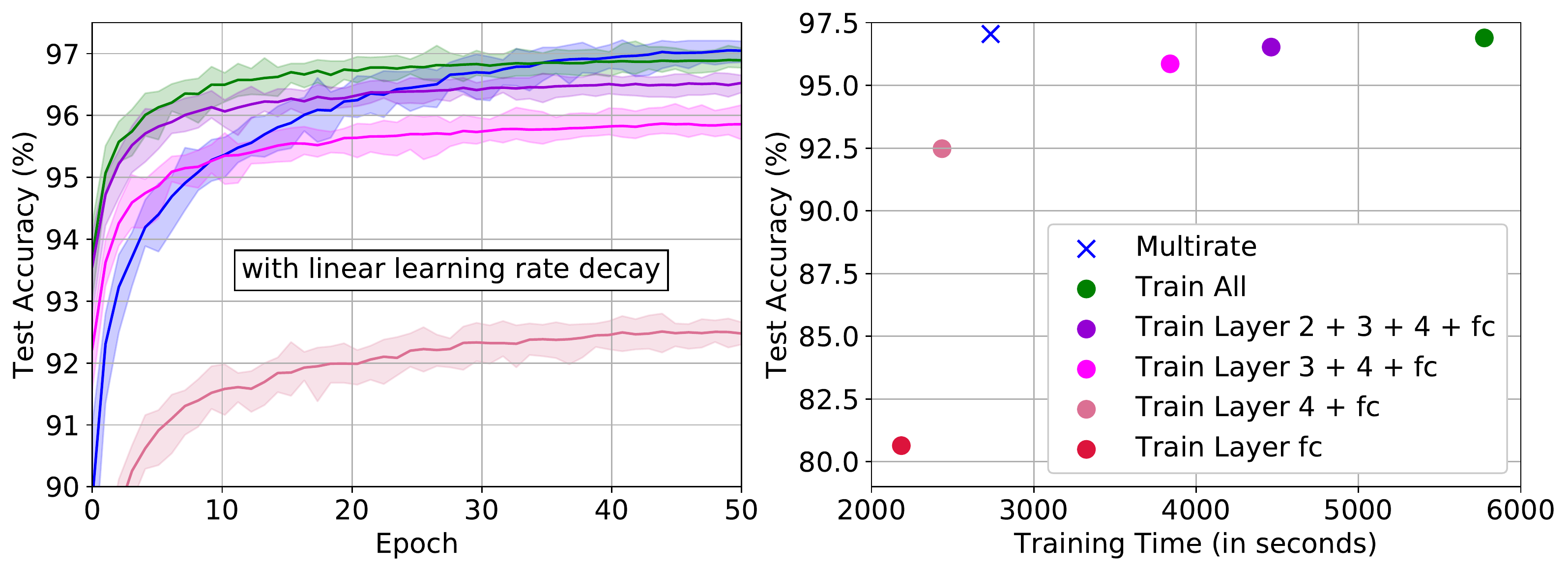}
    \caption{Fine-tuning a ResNet-34 architecture on CIFAR-10 data (same setting as in Figure \ref{ResNet-34_rescaled}), but using linear learning rate decay with initial learning rate set to 5e-3. We again observe that the multirate approach (blue) can be used to train the network in about half the time, while maintaining (or even slightly improving) the test accuracy obtained when fine-tuning the full network (green).} 
    \label{ResNet-34_rescaled_lrdecay}
\end{figure*}

\begin{figure*}
\centering
    \includegraphics[width=0.83\linewidth]{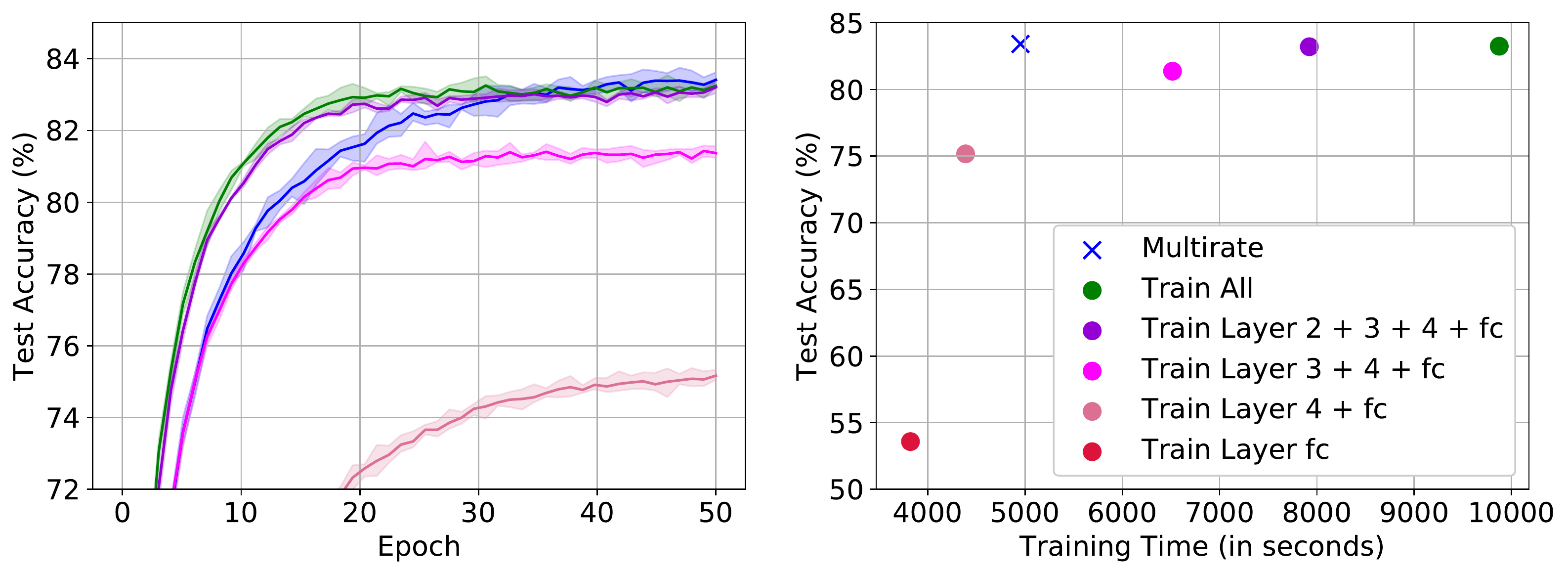}
    \caption{Same set-up as in Figure \ref{ResNet-34_rescaled}, but here we consider a pre-trained ResNet-50 architecture for CIFAR-100 data.
    The highest test accuracy is reached using our multirate approach (blue), which can be used to train the net in almost half the time.}
    \label{ResNet-50_rescaled}
\end{figure*}

\section{A Multirate Approach to Transfer Learning}\label{sec:transferlearning}
We now discuss the application of our multirate scheme in the context of transfer learning, proposing a specific layer-wise division of the model parameters into fast and slow components within Algorithm \ref{multirateSGDwithmom}. We will see that this can significantly reduce the computational cost of fine-tuning.
   
\textbf{Background.} The use of pre-trained deep neural networks has become a popular choice of initialization \cite{BERT,Yosinski2014}. These pre-trained networks are readily available through popular machine learning libraries such as PyTorch \cite{Pytorch}, and are usually trained on large datasets, such as ImageNet for vision applications \cite{Huh2016} or large text corpora for natural language processing \cite{HowardRuder2018}. 
Using a pre-trained network as initialization has been shown to significantly accelerate training and typically improves the generalization performance of the resulting model \cite{Yosinski2014,He2019,Radford2018}.
The procedure is typically as follows: start with a pre-trained model, remove task-specific layers, and then re-train (part of) the network on the new target task. Later layers of neural networks tend to capture more task-specific knowledge, while early layers encode more general features, which can be shared across tasks \cite{Yosinski2014,Hao2019,Neyshabur2020,Raghu2019}. Hence to speed up training (and, in low target-data scenarios, to prevent overfitting), one sometimes does not re-train the full neural network, but only the later layers, in particular the final fully connected layer. This process is called fine-tuning \cite{HowardRuder2018,DaiLe2015}. There exists a delicate balance between computational cost and generalization performance of fine-tuned deep neural network architectures. ``Fine-tuning the whole network usually results in better performance'' \cite{Li2020}, but also increases the computational cost.

\textbf{Methodology.} We propose to split a neural network into two parts, the fully connected layer parameters -- the fast part -- and the other parameters of the deep neural network -- the slow part. The fast part is updated with a stepsize $h/k$, while the other part (the slow part) is only updated every $k$ steps with a stepsize $h$. The slow part is very large compared to the fast part. For example, for a ResNet-34 architecture \cite{resnet}, the fully connected layer parameters (the fast part) only constitute 0.024\% of the total parameters. Because of the way the backpropagation algorithm works, for our fast parameter updates we do not need to compute gradients for the full network, because the fast part is the very last layer of the neural network. This is illustrated schematically in Figure \ref{fig:resnetfastslow}. Assuming that computing the gradients constitutes the largest cost of neural network training, we obtain significant speed-up by only needing to compute the full network gradients every $k$ steps. We show that by choosing an appropriate $k$ we can maintain a good generalization performance for nearly half the computational cost. 

\subsection{Numerical Results}
We study the computational speed-up and generalization performance of Algorithm \ref{multirateSGDwithmom} compared to standard fine-tuning approaches. We consider a ResNet-34 architecture \cite{resnet}, which has been pre-trained on ImageNet \cite{Pytorch}, to classify CIFAR-10 data \cite{cifar10}. The standard procedure is to first replace the final fully connected layer of the architecture, to be able to match the number of classes of the target dataset, and then to retrain either the full architecture on the target set or only some of the bottom layers (with layer we refer to convolutional blocks in this setting). In contrast our multirate approach only updates the final fully connected (fc) layer every step and updates the rest of the 
parameters every 5 steps (we have set $k = 5$ in Algorithm \ref{multirateSGDwithmom}). We use as base algorithm SGD with momentum and performed a hyperparameter
search to select the optimal learning rate for full network fine-tuning. 
We use pre-trained ResNet architectures from PyTorch \cite{Pytorch}.  

We compare our multirate approach (blue) to different fine-tuning approaches in Figure \ref{ResNet-34_rescaled}. Our multirate approach can be used to train the network in almost half the time, without reducing 
the test accuracy 
of the resulting net. We show in Figure \ref{ResNet-34_rescaled_lrdecay} and Figure A\ref{ResNet-34_rescaled_WD} in Appendix \ref{sec:variants} that the same observations hold 
when training using linear learning rate decay or weight decay, respectively.
In Figure \ref{ResNet-50_rescaled} we repeat the experiment for a ResNet-50 architecture (pre-trained on ImageNet), which is fine-tuned on CIFAR-100 data, and observe the same behaviour.

We also test our multirate approach on natural language data and consider a pre-trained DistilBERT (obtained from HuggingFace, transformers library). We fine-tune DistilBERT on SST-2 data \cite{SST} and show the computational speed-up and maintained generalization performance obtained using our multirate approach in Figure \ref{distilbert}. Just as for standard fine-tuning approaches, there exists a trade-off between generalization performance and training time. We find that also including the attention block directly preceding the final fully connected layer into the slow parameters further enhances the generalization performance, without significantly increasing the training time. Results for more GLUE benchmark tasks are provided in Appendix \ref{appx:glue}.

\begin{figure}
        \centering
         \includegraphics[width=0.43\textwidth]{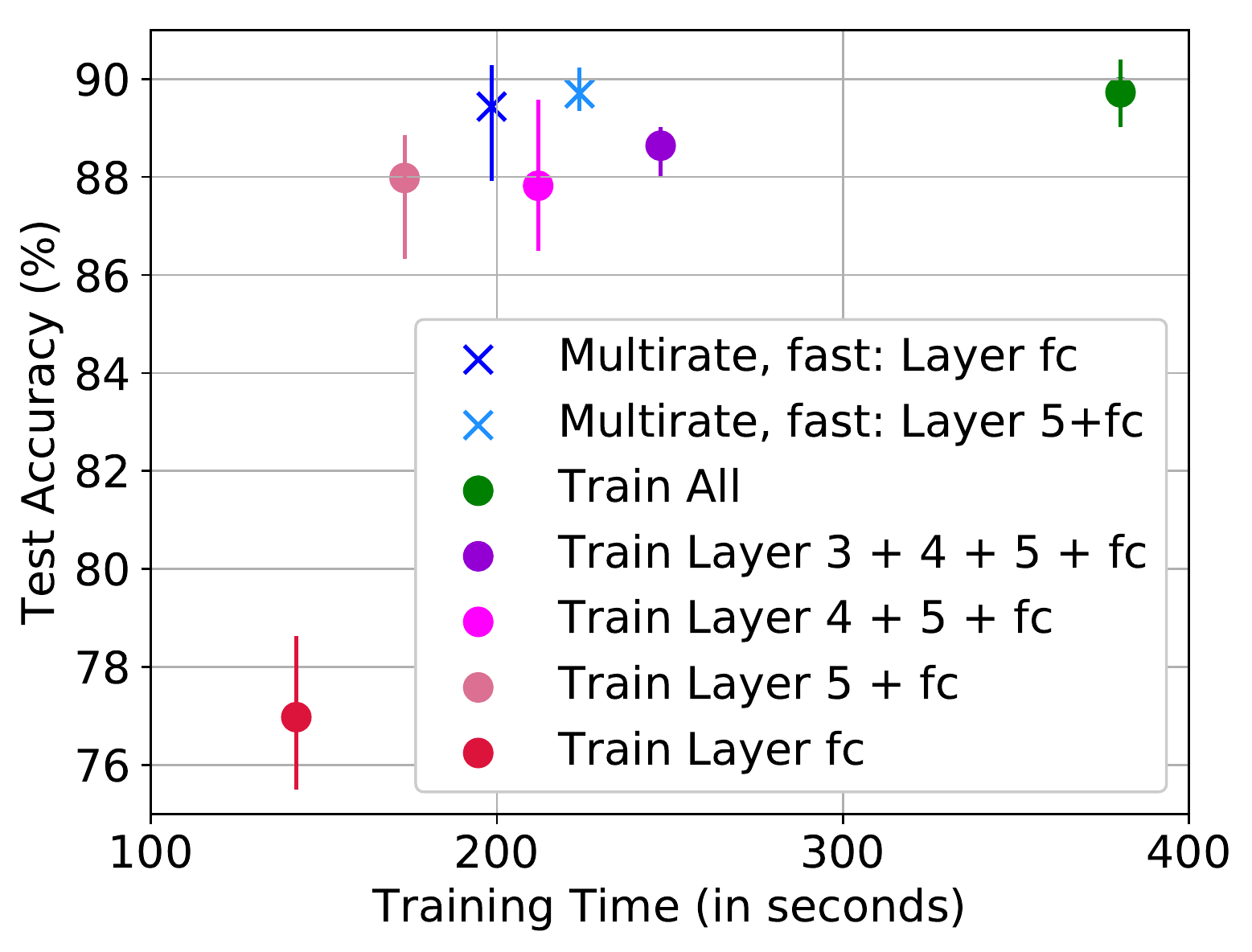}
    \caption{A pre-trained DistilBERT being trained on SST-2 data using different fine-tuning approaches, including our multirate approach (blue). Including the final attention block (together with the fc layer) in the slow parameters gives enhanced generalization performance for limited additional cost (light blue) and lowers the variance across multiple runs. We set $h/k$ = 1e-4, $k = 5,$ $\mu = 0.9$ in Algorithm \ref{multirateSGDwithmom}, batchsize = 16, and average results over 10 runs. \vspace*{-0.4cm} }
    \label{distilbert}
\end{figure}

\subsection{Complexity Analysis}
The number of floating point operations (FLOPs) for a forward pass through a neural network forms the lower bound of the execution time \cite{justus2018}. The number of FLOPs required will depend on the architecture and amount of data, whereas the timing of the FLOPs depends on the hardware used \cite{paleo}. In our case, the number of FLOPs for the forward pass is the same for our multirate algorithm and standard 
approaches. The speed-up we obtain on ResNet architectures and DistilBERT arises from only needing to compute the gradients for the full net every $k$ steps, while computing the gradients for the final fc layer(s) (and optionally the final convolutional/attention block) of the network every backward pass. The backward pass for the multirate method is hence a subset of the backward pass through the full net. Our multirate approach does require storing previous gradients of the slow parameters across iterations, which affects the amount of available memory.

Consider a neural network with $L$ layers and our multirate scheme, where the fast parameters are set to be the final $\ell$ layers of the network with $\ell \ll L$. To get a relative idea of the speed-up obtained using the multirate approach, consider the ratio of the standard forward plus backward pass cost compared to the forward plus backward pass cost for our multirate approach over $k$ steps: 
\begin{align}
    &\frac{\ \ \text{forward + backward pass \ \ \ \ full net \ \ }}{\ \ \text{forward + backward pass \ \ multirate \ \ }}  \\ 
    &= \frac{k L +  \ \ \ \ \ \ \ \ k L\ \ \ \ \ \ \ \ }{kL+L+(k-1)\ell} 
    = \frac{2k L}{(k+1)L+(k-1)\ell}, \ \ \ell \ll L. \nonumber 
\end{align}
For comparison, when only fine-tuning the last $\ell$ layers of the network the 
cost is $k(L+\ell)$, but depending on the choice of $\ell$ this typically results in a lowered generalization performance. If the number of layers $L$ is large, one can obtain a large speed-up using the multirate approach by only having to backpropagate 
through the full network every $k$ steps, while maintaining a similar generalization performance as when fine-tuning the whole network. The exact speed-up obtained depends on the size of the layers and the hardware used. We performed our experiments in PyTorch on NVIDIA DGX-1 GPUs. 
\subsection{Ablation Studies}\label{sec:ablation}
We consider a pre-trained DistilBERT being fine-tuned on SST-2 using our multirate approach (same set-up as in Figure \ref{distilbert}) and perform ablation studies. In Table \ref{tab:distilbert_linearpath} we show that pushing the slow parameters along a linear path (as in Algorithm \ref{multirateSGDwithmom}) improves the test accuracy compared to Algorithm \ref{multirateSGDwithmomwithoutlinear}. We also find that using the same stepsize $h_S = h_F$ for both the fast and slow parameters, but still only updating the slow parameters every $k$ steps, does not lead to the same performance as using larger stepsize $h_S = k \cdot h_F$ for the slow parameters (Table \ref{tab:distilbert_samelr}). 
Finally, we provide a study on the role of $k$ 
(Table A\ref{tab:distilbert_varyk} in Appendix \ref{appx:ablation}). Every epoch the slow parameters only see 1/$k$-th of the minibatches and are updated less frequently with a timestep $k$ times larger than for the fast parameter update. 
We find optimal performance with $k = 5$, although training time can be decreased by choosing larger values of $k$. This trade-off needs to be taken into account when choosing $k$.

\begin{table}[h]
    \centering
        \caption{\textbf{Effect of continuously pushing slow parameters along a linear path.} Same setting as in Figure \ref{distilbert}, where
    fast parameters $\theta_F$ are set to be the fully connected (fc) layer + optionally the final attention block (denoted as layer 5) of a DistilBERT. Results are presented over 10 runs. We compare Algorithm \ref{multirateSGDwithmom} (uses linear drift) to Algorithm \ref{multirateSGDwithmomwithoutlinear}. We find
    that pushing the slow parameters along a linear path during the fast parameter update improves the mean test accuracy. \vspace*{0.2cm}} 
    \label{tab:distilbert_linearpath}
    \begin{tabular}{c||c|c|c|c}
      $\theta_F$ are & \textit{Linear} & \multicolumn{3}{c}{Test accuracy}\\
            Layer & \textit{path?} & Mean & Min & Max  \\ \hline
fc & Yes  & \textbf{89.43\%} & 87.92\% & 90.28\% \\
& No & 88.69\% & 87.53\% & 89.68\% \\ \hline
5 + fc &  Yes   & \textbf{89.70\%} & 89.35\% & 90.23\%\\
& No  & 89.54\% & 88.91\% & 90.44\%\\
    \end{tabular}
\end{table}
\begin{table}[h]
    \caption{\textbf{Same learning rate for fast and slow parameters.} Same setting as in Figure \ref{distilbert} for a DistilBERT. 
    We study the effect of using the same learning rate for both the fast $\theta_F$ and slow $\theta_S$ parameters, but still only updating the slow parameters every $k$ steps. We compare $h_S = h_F =$ 1e-4 vs. using $h_S = k\cdot h_F =$ 5e-4. Results are presented over 10 runs. We observe that using a larger learning rate for the slow parameters aids performance. 
    \vspace*{0.2cm}
    }
    \label{tab:distilbert_samelr}
    \centering
    \begin{tabular}{c||c|c|c|c}
   $\theta_F$ are & \textit{Higher} $h$ & \multicolumn{3}{c}{Test accuracy}\\
            Layer  & \textit{for} $\theta_S$ \textit{?} & Mean & Min & Max \\ \hline
fc & Yes  & \textbf{89.43\%} & 87.92\% & 90.28\% \\
& No & 88.78\% & 87.64\% & 89.90\%\\ \hline
5 + fc &  Yes   & \textbf{89.70\%} & 89.35\% & 90.23\%  \\
& No  & 89.29\% & 88.08\% & 89.95\%\\
    \end{tabular} \\ \ \\
\end{table}

\section{Multirate Training From Scratch}\label{sec:scratch}
Whereas the previous section focused on using a multirate approach to obtain computational speed-up in transfer learning settings, we will now discuss how multirate training can be used to enhance the generalization performance of neural networks trained from scratch. We already illustrated this for the patch-augmented CIFAR-10 data set in Figure \ref{fig:patches}, where a two-scale multirate approach can both memorize the patch and simultaneously obtain good performance on clean data, whereas fixed learning rate approaches fail to do both. In this section we will show the potential of using a multirate approach to regularize neural networks.

\subsection{A Multirate Approach for Neural Network Regularization} \label{sec:randomsubgroups} 
Instead of using layer-wise partitioning, in this section we use randomly selected subsets of the neural network weight matrices (and optionally the biases) to form the slow parameters in Algorithm \ref{multirateSGDwithmom}. Every $k$ optimization steps a different subset of the network parameters is randomly selected. For the technique presented here we slightly modify Algorithm \ref{multirateSGDwithmom}, by setting all the slow parameters $\theta_S$ to be zero during the $k$ fast weight updates.
After the fast parameter update, the slow parameters resume their previous value. They are then updated 
together with the fast weights in a single step, but using a larger time-step $h\cdot k$ for the slow weights. 

The base algorithm we use is again SGD. 
In Figure \ref{fig:randomsubset} we show that our multirate technique can be used to obtain enhanced performance on a single hidden layer perceptron applied to MNIST data. Our technique is  inspired by Dropout \cite{dropout}, although there are some important differences: we do not modify the network architecture and keep the `slow' weights deactivated for multiple steps after which we update them with a larger time-step. We show in Figure \ref{fig:randomsubset} the importance of the multirate aspect (blue), i.e. removing the multirate component from our approach results in worse performance (orange). Further, we compare our technique with dropout in Figure \ref{fig:randomsubsetpenntreebank} for a small transformer trained on the Penn Treebank dataset \cite{PennTreebank} and obtain enhanced validation loss. Ablation studies for $k$ and uncoupled learning rates are provided in Table A\ref{tab:mnist_varyk} and Table A\ref{tab:penntreebank_varyk} in Appendix \ref{appx:ablation}.

\begin{figure}[h]
    \centering
    \includegraphics[width=0.82\linewidth]{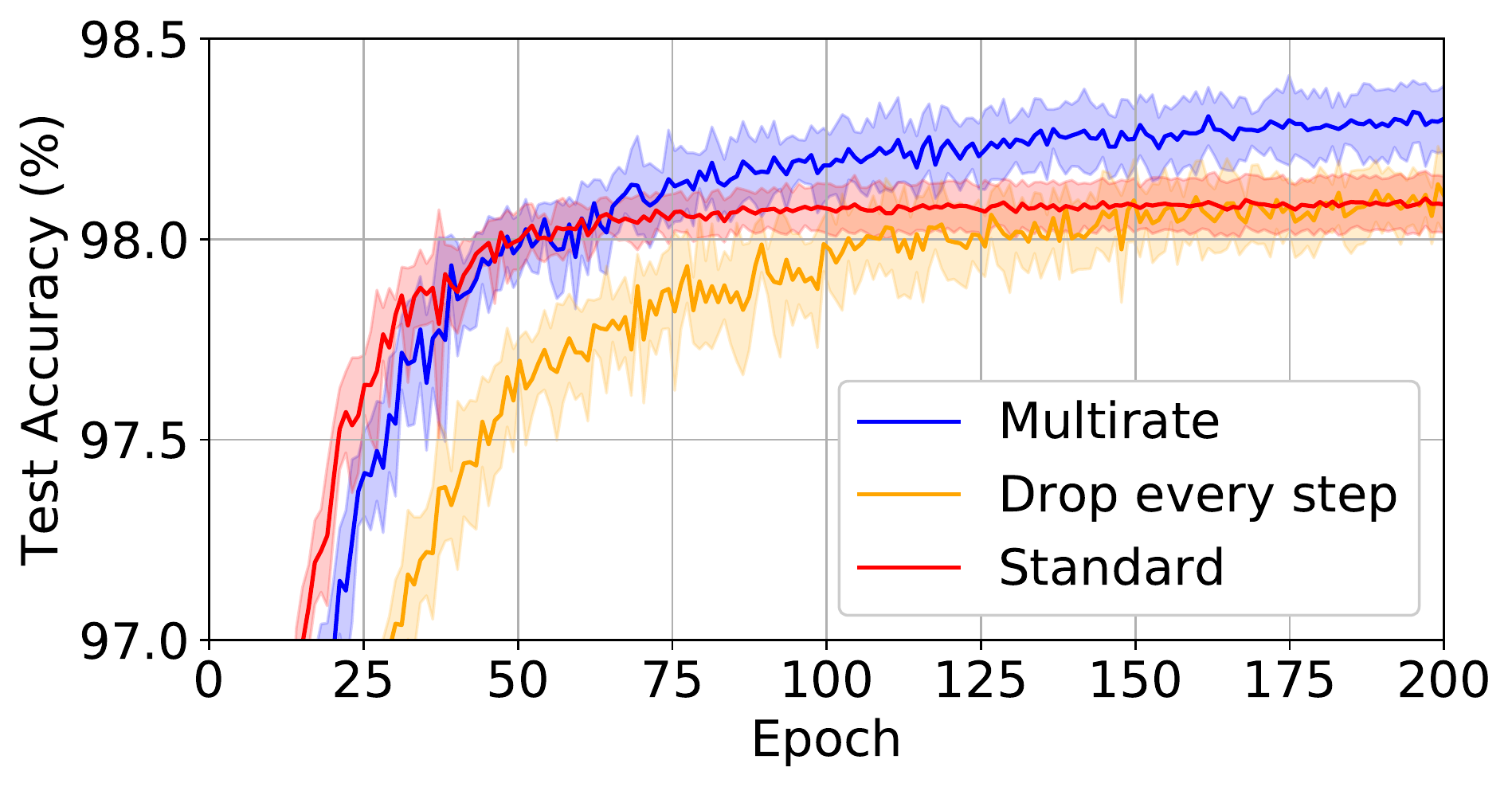}
        \vspace*{-0.35cm}
    \caption{Single hidden layer perceptron trained on MNIST using SGD with $h = 0.1$. Our multirate technique with $k = 5$ (blue) de-activates weights in the input and hidden layer with a probability of 0.8 and 0.5, respectively, and obtains a higher test accuracy than standard SGD (red). We also test removing the multirate component from our approach, which results in an algorithm which sets a different part of the weights to zero every step (orange), and does not perform as well as the multirate technique (blue).} 
    \label{fig:randomsubset} 
\end{figure}
\begin{figure}[h]
    \centering
    \includegraphics[width=0.82\linewidth]{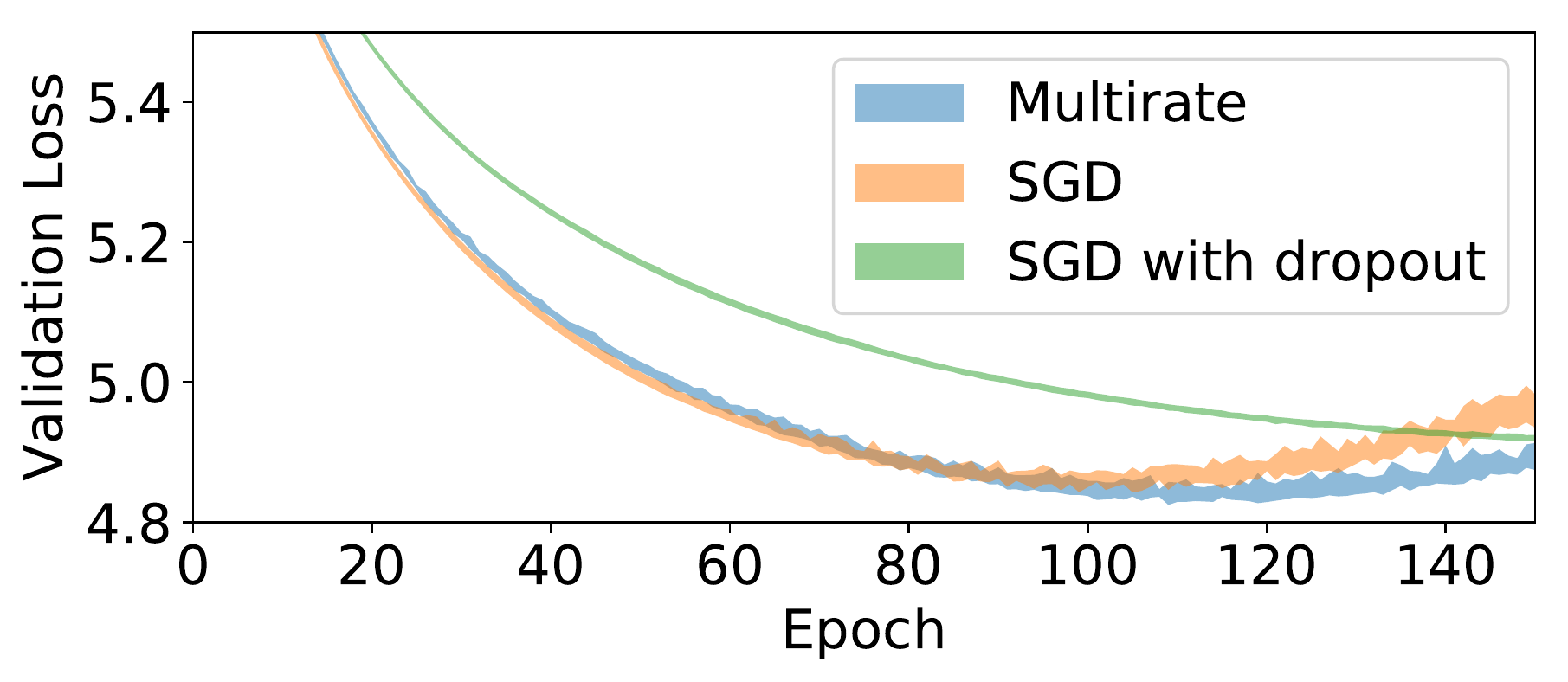}
    \vspace*{-0.35cm}
    \caption{
    A transformer trained on Penn Treebank data \cite{PennTreebank} using SGD with $h = 0.1$ and batchsize 128. Results are provided over 10 runs. The transformer has 2 encoder layers, where each encoder layer consists of self-attention with 2 heads and a feedforward network with 200 nodes followed by layer norms. We compare our multirate technique with $k = 5$ (blue) with vanilla SGD (orange) and SGD with appropriately tuned dropout (green) for the encoder layers. Our multirate approach obtains lower validation loss.}
    \label{fig:randomsubsetpenntreebank}
\end{figure}

\section{Related Work}\label{sec:lit} 
Intuition for using fast and slow weights in a machine learning context can be found in neuroscience, as synapses operate at different time scales. One of the earliest mentions of fast and slow weights in the machine learning literature was by \citet{HintonPlautfastweights}, who set each connection to have both a rapidly changing weight (which was supposed to act as a temporary memory) and a slowly changing weight which stores long-term knowledge. More recently, \citet{Ba2016} used fast weights as a temporary memory to improve recurrent neural networks. 

Our multirate approach to transfer learning (Section \ref{sec:transferlearning}) has 
similarities to multiple time-stepping techniques used in molecular dynamics, such as r-RESPA
\cite{Tuckerman1991, Tuckerman1992}, where the fast dynamics is typically cheap to compute in comparison with the slow dynamics. This is similar to our transfer learning application, where we set the final layer of the net to be the fast part to obtain computational speed-up. Although the use of more refined transfer learning schemes may lead to further test accuracy enhancement, the focus of our approach is to obtain significant computational speed-up, while maintaining the same test accuracy. 
Further, we introduce linear drift of the slow parameters during the fast parameter update. The use of linear drift draws inspiration from the reversible averaging approach to multiple time-stepping by \citet{Leimkuhler2001}, but forms a novel technique for multirate methods. 

Further inspiration arises from the use of partitioned integrators for neural network training. It is well-known that different layers play different roles \cite{Zhang2019} and that later layers capture more task-specific knowledge, while early layers capture more general features, which can be shared across tasks \cite{Yosinski2014,Hao2019,Neyshabur2020,Raghu2019}. It is hence natural to train the different layers of the neural network using layer-wise adaptive learning rates \cite{LARS}, layer-wise large-batch optimization techniques \cite{LAMB}, using different optimizers \cite{LMV}, or by only being Bayesian for certain layers \cite{kristiadi2020,murfet2020}. 

The multirate regularization technique (Section \ref{sec:randomsubgroups}) has similarities to Dropout \cite{dropout} and DropConnect \cite{dropconnect}. 
Dropout can enhance the robustness and generalization performance of neural networks, and is used widely, although its performance in combination with batch normalization \cite{batchnorm} is an ongoing area of research \cite{understandingBN, BNdrop1,BNdrop2}. In contrast to dropout and its variants we do not modify the network architecture but incorporate our technique inside the 
optimizer, which randomly selects a subset of the weights as the slow part and keeps these de-activated for multiple steps. 
These weights are then re-activated and 
updated with a larger time-step, before de-activating a different subset. Instead of making strong claims, in this work we merely aim to illustrate the potential of using multirate techniques as a manner of regularization. We see an exploration of multirate variants of dropout as an exciting avenue for future work. 

\section{Discussion and Future Work}\label{sec:discussion}
We outline possible directions for future work using multirate methods for neural network training along two axes: 1) different splitting choices of the neural network parameters into fast and slow parts and 2) using different optimizers or optimizer hyperparameter settings to train the different partitions. Our methods can be further generalized by combining them with well-known machine learning techniques, such as dropout or by exploring their behaviour under learning rate scheduling.

\textbf{Splitting choices.} For our multirate training approach 
we need to separate the neural network parameters into fast and slow parts. We illustrated the potential of this approach for different parameters splittings. 
In Section \ref{sec:randomsubgroups} we used random subgroups, where we randomly selected a different subset of the network parameters to be the slow parameters every $k$ optimization steps. 
In Section \ref{sec:transferlearning} we used a layer-wise partitioning, where we set the final layer(s)
to be the fast parameters and the remaining parameters to be the slow parameters, for transfer learning applications. 

An interesting direction for future work is to further explore layer-wise splitting when training networks from scratch, e.g., one could separate the early from later layers and train these with different time scales. It is important to note that the computational speed-up we obtained for the transfer learning setting by only computing the gradients for the final layer(s) at every step (Section \ref{sec:transferlearning}), does not easily transfer to training from scratch, where the same approach significantly reduces generalization performance (earlier layers need to be updated more frequently to train well). Although for different choices of the layer-wise splitting the computational speed-up is lost, the use of layer-wise partitioned multirate algorithms may still enhance generalization performance compared to vanilla optimizers. \citet{LARS} found that layer-wise adaptive learning rates can aid training. Further, network layers were shown to have different sensitivities to re-initialization \cite{Zhang2019} and to optimizer hyperparameter settings such as the learning rate \cite{VF}. This motivates training different layers with different initializations or learning rates. 

Another splitting option is to set the biases of a multi-layer perceptron architecture to be the slow parameters, while keeping the weights on the fast time scale. In Appendix \ref{sec:slowbias} we show that using this approach we can obtain higher test accuracies on spiral data and provide ablation studies. 
This illustrates the potential of other parameter splittings.
\ \\ \ \\
\textbf{Hybrid optimization schemes.} For our multirate approach we partition the network into multiple parts which we train on different time scales. A natural extension is to also use different optimizers or optimizer hyperparameters to train the different partitions, e.g., using SGD for the slow part, but SGD with momentum for the fast part(s),
or using sampling techniques such as SGLD or discretized underdamped Langevin dynamics for certain parts. The latter was considered for layer-wise partitionings in \citet{LMV} and \citet{murfet2020}. In this work we used the same base algorithm for all partitions to keep the focus on the role of different time scales.
However, we expect that 
further performance enhancement may be achieved by using hybrid optimization schemes.

\section{Conclusion} This work illustrates the potential of multirate methods for various neural network training applications. In particular, we show that a multirate approach can be used to significantly reduce the computational cost for fine-tuning neural networks, without losing test accuracy or requiring extensive hyperparameter tuning. By introducing the use of multirate techniques to the machine learning community, showing their use in different training settings, and outlining various directions for future work, we hope to have built a strong foundation for further research in this area.
 
 \section*{Acknowledgements}
We thank the reviewers for many helpful comments. The authors wish to thank Katerina Karoni for providing valuable comments on the original proof of Theorem B.4 that led to the creation of the revised version found in Appendix B. During the creation of the paper Tiffany Vlaar was supported by The Maxwell Institute Graduate School in Analysis and its Applications, a Centre for Doctoral Training funded by the UK Engineering and Physical Sciences Research Council (grant EP/L016508/01), the Scottish Funding Council, Heriot-Watt University and the University of Edinburgh.
 
\bibliography{references}
\bibliographystyle{icml2022}

\newpage
\appendix
\onecolumn
\section{Variants of our Multirate Training Algorithms}\label{sec:variants}
Our multirate training scheme partitions the network parameters into multiple components (Algorithm 1 and 2).
This setting lends itself naturally to training the different components (or copies) using different optimization strategies. 

To discuss this more concretely, recall Langevin dynamics from Eq. \eqref{Langevin} in the main paper:
\begin{align*}
  \text{d}\theta_{\alpha} &= p_{\alpha}\ \text{d}t, 
  \\ \text{d}p_{\alpha} &= \tilde{G}_{\alpha}(\theta)\text{d}t-\gamma_{\alpha} p_{\alpha}\  \text{d}t+\sqrt{2\gamma_{\alpha}\tau_{\alpha}}\ \text{d}W_{\alpha}, \ \text{where} \ {\alpha} = F,S,
\end{align*}
with neural network parameters $\theta= (\theta_F,\theta_S)\in \mathbb{R}^n$, momentum $p= (p_F,p_S)\in \mathbb{R}^n$, noisy (due to subsampling) gradient $\tilde{G}_{\alpha}(\theta)$ of the loss with respect to $\theta_{\alpha}$, Wiener process $W$, and hyperparameters $\gamma_{\alpha},\tau_{\alpha} > 0$. Using discretized Langevin dynamics to train neural networks allows for incorporation of both momentum and additive noise, the size of which is controlled by the $\gamma_{\alpha}$ and $\tau_{\alpha}$ hyperparameters, respectively. A straightforward variant is thus to use different values for $\gamma_{\alpha}$ and/or $\tau_{\alpha}$ for the fast and slow components.
The temperature hyperparameter $\tau_{\alpha}$ controls the driving noise and thus the transition between a pure optimization and sampling approach. When small it can benefit neural network optimization \cite{LMV,coldposterior}. Using small values of $\tau$ for parts of the dynamics that require further exploration may thus benefit training. On the other hand, to train a component with stochastic gradient descent with momentum one can set $\tau = 0$. An example of a possible combination of Langevin dynamics with additive noise for the fast dynamics and without additive noise (corresponding to SGD) for the slow dynamics is then:
\begin{align*}
  \text{d}\theta_{F} &= p_{F}\ \text{d}t, \ \text{d}p_{F} = \tilde{G}_{F}(\theta)\text{d}t-\gamma_{F} p_{F}\  \text{d}t+\sqrt{2\gamma_F\tau}\ \text{d}W_{F} \\
    \text{d}\theta_{S} &= p_{S}\ \text{d}t, \ \text{d}p_{S} = \tilde{G}_{S}(\theta)\text{d}t-\gamma_{S} p_{S}\  \text{d}t.
\end{align*}
Equivalently, one could change the value of the momentum hyperparameter $\gamma_{\alpha}$ for the different partitionings. Or use different optimizers for the different components, such as Adam and SGD. Finally, it would be interesting to study the effect of using different size initializations for the different components, which essentially starts off the different components on different scales. Of course, any of these suggestions require extra tuning of the algorithm, which is why we focused on SGD with the same momenta values and initializations for all components in the paper. 
We expect however that using hybrid optimization schemes may lead to even further performance enhancement, which we aim to explore in future work. 

\textbf{Weight decay.} As a simple extension of Algorithm 1 we provide the case with weight decay in Algorithm A\ref{multirateSGDwithmomWD}, where we have used $\omega$ to denote the amount of weight decay. Our implementation of weight decay is the same as used in PyTorch for SGD \cite{Pytorch}. Due to the use of linear drift all parameters are trained using the same learning rate (although slow parameter gradients are only updated every $k$ steps), which is why no extra tuning of the weight decay is needed. We show in Figure A\ref{ResNet-34_rescaled_WD} for a pre-trained ResNet-34 being fine-tuned on CIFAR-10 data  (same set-up as in Figure \ref{ResNet-34_rescaled}) that using Algorithm A\ref{multirateSGDwithmomWD} our multirate approach can train the network in almost half the time, without reducing the test accuracy.

\makeatletter
\renewcommand{\fnum@figure}{\figurename~A\thefigure}
\makeatother

\begin{figure}[h]
\centering
    \includegraphics[width=0.8\linewidth]{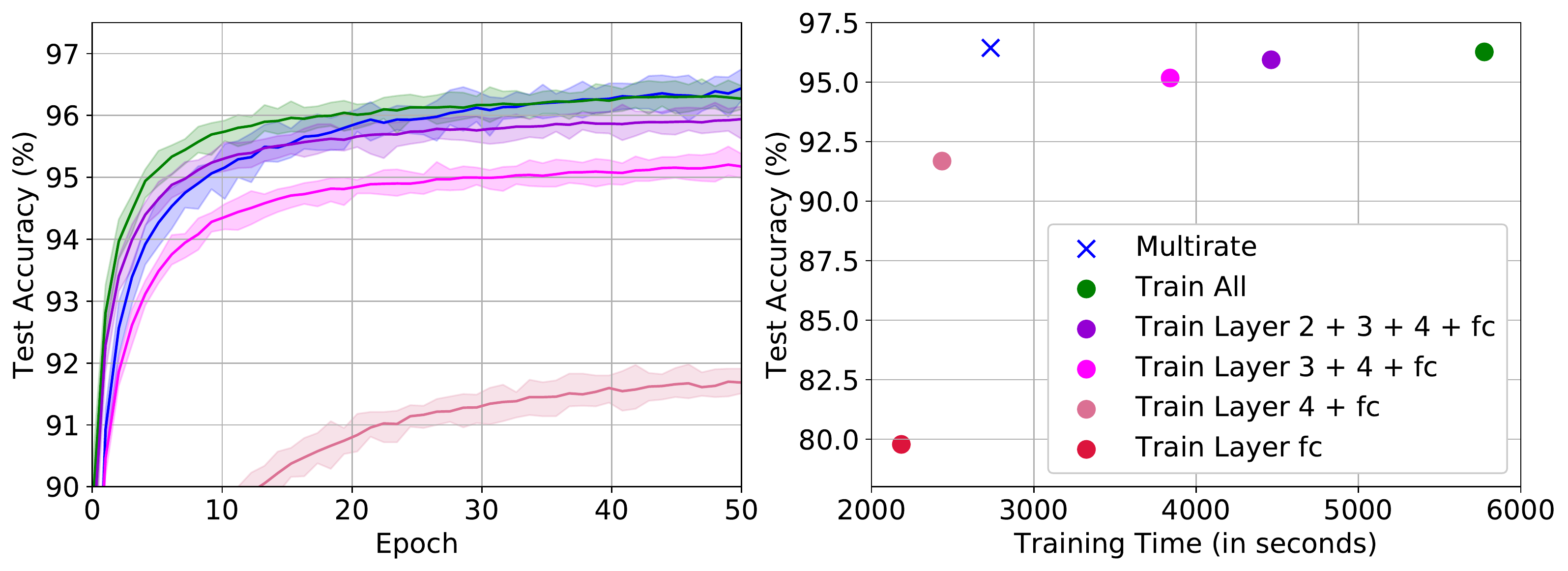}
    \vspace*{-0.35cm}
    \caption{A pre-trained ResNet-34 being trained on CIFAR-10 data (same set-up as in Figure \ref{ResNet-34_rescaled}) using different fine-tuning approaches and our multirate approach Algorithm A\ref{multirateSGDwithmomWD} (blue).  Results are averaged over 20 runs and approaches are trained using SGD with momentum with weight decay. We set $h/k = 0.001, k = 5, \mu = 0.9$, and $\omega_S = \omega_F =  $ 5e-4  in Algorithm A\ref{multirateSGDwithmomWD}. Our multirate approach (blue) can be used to train the net in almost half the time, while maintaining generalization performance. Typical fine-tuning approaches only train the bottom layers of the network, e.g., layer 4 + fc, which results in a comparable speed-up, but much lower test accuracy.} 
    \label{ResNet-34_rescaled_WD}
\end{figure}

\makeatletter
\renewcommand{\ALG@name}{Algorithm A\hspace*{-0.1cm}}
\makeatother
\begin{algorithm}[H] 
 \caption{Multirate SGD with linear drift and weight decay} 
 \label{multirateSGDwithmomWD}
\begin{algorithmic}
\STATE $p_S := \mu p_S + \nabla_{\theta_S} \mathcal{L}(\theta_S,\theta_F)+\omega_S\theta_S $
\FOR{$i=1,2,...,k$}
\STATE $p_F := \mu p_F + \nabla_{\theta_F}\mathcal{L}(\theta_S,\theta_F) +\omega_F \theta_F $
\STATE $\theta_F := \theta_F-\frac{h}{k}p_F$
\STATE $\theta_S := \theta_S-\frac{h}{k}p_S$ 
\ENDFOR
\end{algorithmic}
\end{algorithm}
\section{Convergence Analysis}\label{appx:conv}
Recall our main assumptions:
\begin{assumption} We assume function $f: \mathbb{R}^n\rightarrow \mathbb{R}$ to be $L$-smooth, i.e., $f$ is continuously differentiable and its gradient is Lipschitz continuous with Lipschitz constant $L > 0$:
\begin{align}
    \|\nabla f(\varphi)-\nabla f(\theta)\|_2\leq L\|\varphi-\theta\|_2, \ \forall \theta,\varphi\in\mathbb{R}^n.
\end{align}\label{Aassm:Lsmooth}
\end{assumption}
\begin{assumption}
We assume that the second moment of the stochastic gradient is bounded above, i.e., there exists a constant $M$ for any sample $x_i$ such that
\begin{align}
    \|\nabla f_{x_i}(\theta)\|^2_2 \leq M, \   \ \forall \theta \in \mathbb{R}^n.
\end{align}\label{Aassm:boundedvariance}
\end{assumption}
Assumption \ref{Aassm:boundedvariance} guarantees the variance of the stochastic gradient to be less than $M$, because
\begin{align}
    \text{Var}(\nabla f_{x_i}(\theta)) &= \mathbb{E}\|\nabla f_{x_i}(\theta)-\mathbb{E}[\nabla f_{x_i}(\theta)]\|^2_2  \nonumber \\
    &= \mathbb{E}\|\nabla f_{x_i}(\theta)-\nabla f(\theta)\|^2_2 \nonumber \\
    &= \mathbb{E}\|\nabla f_{x_i}(\theta)\|^2_2-\|\nabla f(\theta)\|^2_2 
\end{align}
where we used $\mathbb{E}[\nabla f_{x_i}(\theta)] = \nabla f(\theta)$ (unbiased gradient) for the second equality and Var$(X) = \mathbb{E}[(X-\mathbb{E}[X])^2] = \mathbb{E}[X^2]-\mathbb{E}[X]^2$.

If we assume Assumption \ref{Aassm:Lsmooth} holds, we obtain the following Lemma, which we will need for the proof of the main theorem:
\begin{lemma}\label{lemma:fromLsmooth}
If $f:\mathbb{R}^n\rightarrow \mathbb{R}$ is $L$-smooth then $\forall \theta,\varphi\in\mathbb{R}^n$:
\begin{align}
    |f(\varphi)-(f(\theta)+\nabla f(\theta)^T(\varphi-\theta))|\leq \frac{L}{2}\|\varphi-\theta\|^2_2.
\end{align}
\end{lemma}
\begin{proof}[Proof of Lemma \ref{lemma:fromLsmooth}.]
From the fundamental theorem of calculus:
\begin{align*}
    \int_0^1 \nabla f[\theta+t(\varphi-\theta)]^T (\varphi-\theta) \ \text{d}t = f(\varphi)-f(\theta)
\end{align*}
So using Cauchy-Schwartz and the assumption that $f$ is $L$-smooth we obtain:
\begin{align*}
    | f(\varphi)-f(\theta)-\nabla f(\theta)^T(\varphi-\theta) | &= \left | \int_0^1 \left ( \nabla f[\theta+t(\varphi-\theta)] -\nabla f(\theta)\right )^T (\varphi-\theta) \ \text{d}t \right | \\
    &\leq \int_0^1 \| \nabla f[\theta+t(\varphi-\theta)] -\nabla f(\theta)\|_2 \| \varphi-\theta \|_2 \text{d}t \\
    &\leq L \|\varphi-\theta\|^2_2\int^1_0 t \ \text{d}t = \frac{L}{2}\|\varphi-\theta\|^2_2.
\end{align*}
\end{proof}
As a starting point for our layer-wise multirate approach we partition the parameters as $\theta = \{\theta_F,\theta_S\}$, with $\theta_F
\in \mathbb{R}^{n_F}, \theta_S
\in \mathbb{R}^{n_S}$, $n = n_F +n_S$. The multirate method update for base algorithm SGD is 
\begin{align}
    \theta^{t+1}_{\ell} = \theta^t_{\ell}-h \nabla f_{\ell,x_i}(\theta^t),\label{eq:multirate}
\end{align}
where $\ell\in\{F,S\}$, $\theta^{t}_{\ell}$ are the parameter groups at iteration $t$, $h$ is the stepsize, and $\nabla f_{\ell,x_i}$ denotes the gradient of the loss of the $i$th training example for parameters 
$\theta^{t}_{\ell}$, 
where $\nabla f_{F,x_i}(\theta^t) = \nabla f_{F,x_i}(\theta^t)$ and with linear drift: for any $t \in [\tau,\tau+k-1]$, where $\tau$ is divisible by $k$, $\nabla f_{S,x_i}(\theta^t) = \nabla f_{S,x_i}(\theta^{\tau})$. 
The total number of iterations $T$ is always set to be a multiple of $k$.
\ \\ \ \\ In the following we  denote $\nabla f_{x_i}(\theta^t) = \{ \nabla f_{F,x_i}(\theta^t),\nabla f_{S,x_i}(\theta^t) \}$ and 
$g_{x_i}(\theta^t) = \{\nabla f_{F,x_i}(\theta^t),\nabla f_{S,x_i}(\theta^\tau)\}$, such that the parameter update rule becomes
\begin{align}
    \theta^{t+1} = \theta^t - hg_{x_i}(\theta^t).
\end{align}

Now we want to prove Theorem \ref{thm:conv} in the main body of the paper:
\begin{theorem}\label{th:mainconv}
Assume that Assumptions \ref{Aassm:Lsmooth} and \ref{Aassm:boundedvariance} hold. Then
\begin{align}
     \frac{1}{T}\sum^{T-1}_{t=0} \mathbb{E}\left [ \|\nabla f(\theta^t)\|^2_2 \right ] \leq \frac{ 2(f(\theta^0)-f(\theta^*))}{hT}+ h L M\ell \left ( \frac{1}{3} h L k^2 + 1 \right ),
\end{align}
where $\theta^*$ is the optimal solution to $f(\theta)$. 
\end{theorem}

\begin{proof}[Proof of Theorem \ref{th:mainconv}]
Because $f$ is $L$-smooth, from Lemma \ref{lemma:fromLsmooth} it follows that:
\begin{align}
    f(\theta^{t+1})&\leq f(\theta^t)+\nabla f(\theta^t)\cdot(\theta^{t+1}-\theta^t)+ \frac{L}{2}\|\theta^{t+1}-\theta^t\|^2_2 \nonumber\\
    &\leq f(\theta^t)-h\nabla f(\theta^t) \cdot g_{x_i}(\theta^t)+ \frac{h^2 L}{2}\left \|g_{x_i}(\theta^t)\right \|^2_2 
\end{align}
Taking the double expectation on both sides gives (because of unbiased gradient $\mathbb{E}_{x_i\sim p(X)}[g_{x_i}(\theta^t)] = g(\theta^t)$ and Assumption \ref{Aassm:boundedvariance}):
\begin{align*}
\mathbb{E}[f(\theta^{t+1})-f(\theta^t)] &\leq -h\mathbb{E}\left [\nabla f(\theta^t) \cdot g(\theta^t)\right ]+ \frac{h^2 L M \ell}{2} 
\end{align*}
for number of parameter groups $\ell$ and where $\mathbb{E}[..]$ is the expectation with respect to the parameters. So in $T$ iterations
we have $\theta^T$ such that (using a telescoping sum):
\begin{align}
        f(\theta^*)-f(\theta^0) &\leq \mathbb{E}[f(\theta^T)]-f(\theta^0) \nonumber \\
        &\leq -h \underset{\mathcal{A}}{\underbrace{\sum^{T-1}_{t=0}\mathbb{E}\left [\nabla f(\theta^t) \cdot g(\theta^t)\right ]}}+ 
        \frac{h^2 L M \ell}{2}T. \label{eq:multiratesumoverT} 
\end{align}
For term $\mathcal{A}$ we get 
\begin{align}
    \mathcal{A} = \sum_{t=0}^{T-1} a_t = \sum_{t=0}^{k-1}a_t+\sum_{t=k}^{2k-1}a_t+\dots+\sum^{\tau+k-1}_{t=\tau} a_t+\dots+\sum_{t=T-k}^{T-1}a_t,
\end{align}
where $\sum^{\tau+k-1}_{t=\tau} a_t$ is given by
\begin{align*}
    \sum^{\tau+k-1}_{t=\tau}\mathbb{E}\left [\nabla f(\theta^t)
   \cdot g(\theta^t)\right ] &= \sum^{\tau+k-1}_{t=\tau}\mathbb{E}\left [\{\nabla f_F(\theta^t), \nabla f_S(\theta^t)\} \cdot \{\nabla f_F(\theta^t), \nabla f_S(\theta^\tau)\}\right ] \\
           &= \sum^{\tau+k-1}_{t=\tau} \mathbb{E}\left [\|\nabla f_{F}(\theta^t)\|^2_2 \right ]+ \sum^{\tau+k-1}_{t=\tau}\mathbb{E}[\nabla f_{S}(\theta^t) \cdot (\nabla f_{S}(\theta^{\tau})-\nabla f_{S}(\theta^t)+\nabla f_{S}(\theta^t))]\\
        &= \sum^{\tau+k-1}_{t=\tau} \mathbb{E}\left [\|\nabla f(\theta^t)\|^2_2 \right ]+ \underset{\mathcal{B}}{\underbrace{\sum^{\tau+k-1}_{t=\tau}\mathbb{E}[\nabla f_{S}(\theta^t)\cdot (\nabla f_{S}(\theta^{\tau})-\nabla f_{S}(\theta^t))]}}.
\end{align*}
Because $xy \leq \frac{1}{2}\|x\|^2_2+\frac{1}{2}\|y\|^2_2$ (combination of Cauchy-Schwarz and Young's inequality) (gives 1st inequality) and Assumption \ref{Aassm:Lsmooth} (gives 2nd inequality) we get for term $\mathcal{B}$:
\begin{align*}
    \mathcal{B} &\leq
    \frac{1}{2} \sum^{\tau+k-1}_{t=\tau}\mathbb{E}\left [\|\nabla f_{S}(\theta^t)\|^2_2\right ] +\frac{1}{2}\sum^{\tau+k-1}_{t=\tau}\mathbb{E} \left [\|\nabla f_{S}(\theta^{\tau})-\nabla f_{S}(\theta^t)\|^2_2 \right ]\\
          &\leq \frac{1}{2}\sum^{\tau+k-1}_{t=\tau} \mathbb{E} \left [ \|\nabla f_{S}(\theta^t)\|^2_2 \right ] +\frac{L^2}{2}\mathbb{E}\Bigg[\underset{\mathcal{C}}{\underbrace{\sum^{\tau+k-1}_{t=\tau+1} \|\theta^{\tau}-\theta^t\|^2_2}}\Bigg].
\end{align*}
We get for term $\mathcal{C}$ from Eq. \eqref{eq:multirate} (gives 2nd equality), $\|a_1 + \dots + a_m\|^2_2 \leq m (\|a_1\|^2_2 + \dots + \|a_m\|^2_2)$ 
(gives 1st inequality),
Assumption \ref{Aassm:boundedvariance} (gives 2nd inequality), and $k > 1$ (final inequality):
\begin{align*}
   \mathcal{C} &=  \|\theta^{\tau}-\theta^{\tau+1}\|^2_2 + \|\theta^{\tau}-\theta^{\tau+2}\|^2_2 + \dots +  \|\theta^{\tau}-\theta^{\tau+k-1}\|^2_2 \\
     &= h^2 \left( \left \|
    g_{x_i}(\theta^{\tau})\right \|^2_2 + \left \| g_{x_i}(\theta^{\tau}) +  g_{x_i}(\theta^{\tau+1})\right \|^2_2 
     + \dots +  \left \| g_{x_i}(\theta^{\tau}) + \dots +  g_{x_i}(\theta^{\tau+k-2}) \right \|^2_2  \right ) \\
     &\leq h^2 \left ( \sum_{m=1}^{k-1} m \left \| g_{x_i}(\theta^{\tau}) \right \|^2_2 + \sum_{m=2}^{k-1} m \left \|\ g_{x_i}(\theta^{\tau+1})  \right \|^2_2 
   + \dots + (k-1) \left \| g_{x_i}(\theta^{\tau+k-2})  \right \|^2_2 \right) \\
  &\leq h^2 M \ell \left ((k-1)^2+(k-2)^2+\dots+1 \right ) = h^2 M \ell  \sum_{m=1}^{k-1} m ^2 = h^2 M \ell \left ( k/6 - k^2/2 + k^3/3 \right ) \leq h^2 M \ell k^3/3.
\end{align*}
So overall for term $-h \mathcal{A}$ we get
\begin{align}
    -h \sum^{T-1}_{t=0}\mathbb{E}[\nabla f(\theta^t)\cdot g(\theta^t)]
&\leq -h  \sum^{T-1}_{t=0}\mathbb{E}\left [\|\nabla f(\theta^t)\|^2_2\right] +  h \left | \sum_{\tau} \mathcal{B} \right |   \nonumber \\
    &\leq -\frac{h}{2} \sum^{T-1}_{t=0} \mathbb{E}\left [\|\nabla f(\theta^t)\|^2_2\right] + \frac{1}{6} h^3 L^2 M \ell k^2 T.
\end{align}
Substituting this into Eq. \eqref{eq:multiratesumoverT} gives:
\begin{align}
            f(\theta^*)-f(\theta^0) &\leq \mathbb{E}[f(\theta^T)]-f(\theta^0) \nonumber \\
        &\leq - \frac{h}{2} \sum^{T-1}_{t=0} \mathbb{E}\left [ \|\nabla f(\theta^t)\|^2_2 \right ] + \frac{1}{6} h^3 L^2 M \ell k^2 T+ 
        \frac{h^2 L M \ell}{2}T \nonumber \\
        &= - \frac{h}{2} \sum^{T-1}_{t=0} \mathbb{E}\left [ \|\nabla f(\theta^t)\|^2_2 \right ] + \frac{1}{2}h^2L M\ell T \left ( \frac{1}{3} h L k^2 + 1 \right ).
\end{align}
This gives Theorem \ref{th:mainconv}
\begin{align}
    \frac{1}{T}\sum^{T-1}_{t=0} \mathbb{E}\left [ \|\nabla f(\theta^t)\|^2_2 \right ] \leq \frac{ 2(f(\theta^0)-f(\theta^*))}{hT}+ h L M\ell \left ( \frac{1}{3} h L k^2 + 1 \right ). \nonumber
\end{align}
\end{proof}
For comparison, the convergence analysis for vanilla SGD with fixed stepsize $h$ update
\begin{align}
    \theta^{t+1} = \theta^t-h \nabla f_{x_i}(\theta^t),
\end{align}
where $\nabla f_{x_i}$ denotes the gradient of the loss of the $i$th training example, gives Theorem \ref{th:sgdconv}.

\begin{theorem}\label{th:sgdconv}
Assume that Assumptions \ref{Aassm:Lsmooth} and \ref{Aassm:boundedvariance} hold. Then:
\begin{align}
  \frac{1}{T} \sum^{T-1}_{t=0} \mathbb{E}\left [ \|\nabla f(\theta^t)\|^2_2 \right ] \leq \frac{2(f(\theta^0)-f(\theta^*))}{hT} + \frac{hLM}{2}
\end{align}
where $\theta^*$ is the optimal solution to $f(\theta)$.
\end{theorem}

\begin{proof}[Proof of Theorem \ref{th:sgdconv}.]
Because $f$ is $L$-smooth, from Lemma 1 it follows that:
\begin{align}
    f(\theta^{t+1})-f(\theta^t) &\leq \nabla f(\theta^t)^{Tr}(\theta^{t+1}-\theta^t)+ \frac{L}{2}\|\theta^{t+1}-\theta^t\|^2_2 \nonumber\\
    &\leq-h \nabla f(\theta^t)^{Tr}\nabla f_{x_i}(\theta^t)+ \frac{h^2 L}{2}\|\nabla f_{x_i}(\theta^t)\|^2_2\nonumber
\end{align}
Taking the expectation on both sides gives (because of assumption 2 and unbiased gradient $\mathbb{E}[\nabla f_{x_i}(\theta)] = \nabla f(\theta)$):
\begin{align}
\mathbb{E}[f(\theta^{t+1})-f(\theta^t)] &\leq -h \mathbb{E}\left [ \|\nabla f(\theta^t)\|^2_2 \right ]+ \frac{h^2 L}{2}M
\end{align}
So in $T$ gradient steps we have $\theta^T$ such that:
\begin{align}
        f(\theta^*)-f(\theta^0) \leq \mathbb{E}[f(\theta^T)]-f(\theta^0) \leq - h \sum^{T-1}_{t=0}\mathbb{E}\left [\|\nabla  f(\theta^t)\|^2_2 \right ]+ \frac{h^2 L M T}{2} 
\end{align}
This gives:
\begin{align}
  \frac{1}{T} \sum^{T-1}_{t=0} \mathbb{E}\left [ \|\nabla f(\theta^t)\|^2_2 \right ] \leq \frac{2(f(\theta^0)-f(\theta^*))}{hT} + \frac{hLM}{2}
\end{align}
\end{proof}

\section{Further Experimental Details}
We run our experiments (unless indicated otherwise) with SGD with momentum set to 0.9.
The learning rate varied per experiment and is detailed in the captions of the figures. For the transfer learning experiments (Section 4) it was set to $h = 0.001$ for the ResNet architectures and to $h $ = 1e-4 for the DistilBERT and we did not use weight decay (except for Figure A\ref{ResNet-34_rescaled_WD}). The models were pre-trained on ImageNet, so we resized the CIFAR-10/CIFAR-100 images before training, e.g. as in \citet{revisitresnet}. In Algorithm 1 we set $k = 5$ and varied our partitionings of the network parameters into fast and slow parts. 
All our experiments were run in PyTorch using NVIDIA GPUs. We will discuss specific experiments that require further details below. 

\makeatletter
\renewcommand{\fnum@figure}{\figurename~A\thefigure}
\makeatother

\makeatletter
\renewcommand{\fnum@table}{\tablename~A\thetable}
\makeatother

\subsection{Patch-augmented CIFAR-10}\label{appx:patch}
The patch-augmented CIFAR-10 dataset that we used for Figure \ref{fig:patches} was adapted from the paper by \citet{Li2019}. 
To generate the dataset the 50000 CIFAR-10 training images are split into 10000 patch-free images and 40000 images which contain only a patch with probability 0.2 and contain a patch mixed with CIFAR-10 data with probability 0.8. The $7\times 7$ pixel patch is located in the center of the images. Following \citet{Li2019} to generate the patch, sample $z\sim\mathcal{N}(0,1.5625), a$ a random float in [0,1), 
and $\zeta_i \sim [-0.1,0.1]$ for classes $i = 1,\dots,10$. Then for patch-only images belonging to class $i$ set everything to 0 and add $z\pm1.75a\zeta_i$. 
To generate images containing both a patch and CIFAR-10 data add $z\pm \zeta_i$. For the multirate training approach we partitioned a composite network system into two parts, where each subnetwork was trained on a different timescale. The weights sampled from both parts were averaged and merged every $k$ steps. The exact same learning rates were used as in the original paper by \citet{Li2019}, so $h_F = 0.004$, $h_S = 0.1$, and thus $k = h_S/h_F = 25$.

\section{Additional Experiments}

\subsection{GLUE Tasks}\label{appx:glue}
In Table A\ref{tab:glue} we provide results for fine-tuning a DistilBERT  (same setting as in Section \ref{sec:transferlearning}, Figure \ref{distilbert}) on more tasks from the General Language Understanding Evaluation (GLUE) benchmark \cite{glue}. We compare the performance and computational speed-up of multirate Algorithm \ref{multirateSGDwithmom} with full net fine-tuning. We see that a similar generalization performance is maintained using the multirate approach, while achieving computational speed-up. The focus of the experiment is solely on showing relative computational speed-up, not on beating state-of-the-art.
For these experiments (and in the rest of this paper) we use as base algorithm SGD with momentum. 
However, adaptive optimizers such as Adam tend to be the method of choice in the natural language processing literature \cite{BERT,roberta,AdamvsSGD} and may lead to further performance enhancements. We see exploration of a multirate approach in this setting as an interesting direction for future work. 

\begin{table}[h]
\caption{Performance on dev sets (median over 5 runs) and averaged wallclock time per fine-tuning training run of a 
pre-trained DistilBERT on 
some GLUE tasks. For MNLI we report accuracy on  matched and mismatched sets. Hyperparameter settings: batchsize = 16 (all, except 32 for MNLI), 
$k = 2$ and weight decay set to 5e-4 (MNLI, QNLI), $k = 4$ (RTE, WNLI),
$h/k$ = 3e-3 (MNLI), $h/k$ = 5e-4 (QNLI, RTE), $h/k$ = 1e-4 (SST-2, WNLI) 
and the fast parameters being the linear head and attention block 5 (MNLI, SST-2, WNLI) + attention block 4 (QNLI, RTE).\vspace*{0.2cm}}
    \centering
    \begin{tabular}{ll||c|c|c|c|c}
    & & MNLI & QNLI & RTE & SST-2 & WNLI  \\ \hline
  Accuracy (\textit{\%}) & Full net fine-tuning & 75.3 / 76.7 & 86.6 & 57.8 & 89.7 & 54.2 \\
              & Multirate  & 75.4 / 76.7 & 86.4 & 57.8 & 89.7 & 54.9 \\ \hline 
    Timing (\textit{sec.}) & Full net fine-tuning & 14617 & 14148 & 192 & 380 & 20\\
    & Multirate  & 10797 & 10602 & 127 & 224 & 11
    \end{tabular}
    \label{tab:glue}
\end{table}

\subsection{Slow Biases}\label{sec:slowbias}
We study the effect of putting all the biases of a neural network on the slow time scale, while keeping the weights on the fast time scale and only updating the slow parameters every $k$ steps using Algorithm \ref{multirateSGDwithmom}. Surprisingly, this gives big performance improvements on 4-turn spiral data (adapted from \citet{LMV}) as shown in Figure A\ref{biasslow_spiral}. Figure A\ref{fig:comparebiasslow_spiral} confirms that this enhanced performance is caused by our multirate technique, i.e., freezing the biases for $k$ steps and then boosting them with a larger time-step. Simply putting the biases on a different time scale or freezing the biases for $k$ steps and using the same time-step does not lead to the same performance improvement. In Figure A\ref{fig:whichbias} we show that this effect is caused by the input layer biases, in particular. This seems to suggest a possible connection with data normalization (and the lack thereof for the spiral dataset).  In Figure A\ref{biasslow_resnet} we show that for a ResNet-34 on CIFAR-10 data one also obtains a small performance improvement by using slow biases for the fully connected layer, especially when no data augmentation is used. 

\begin{figure}[h]
    \centering
    \includegraphics[width=0.9\linewidth]{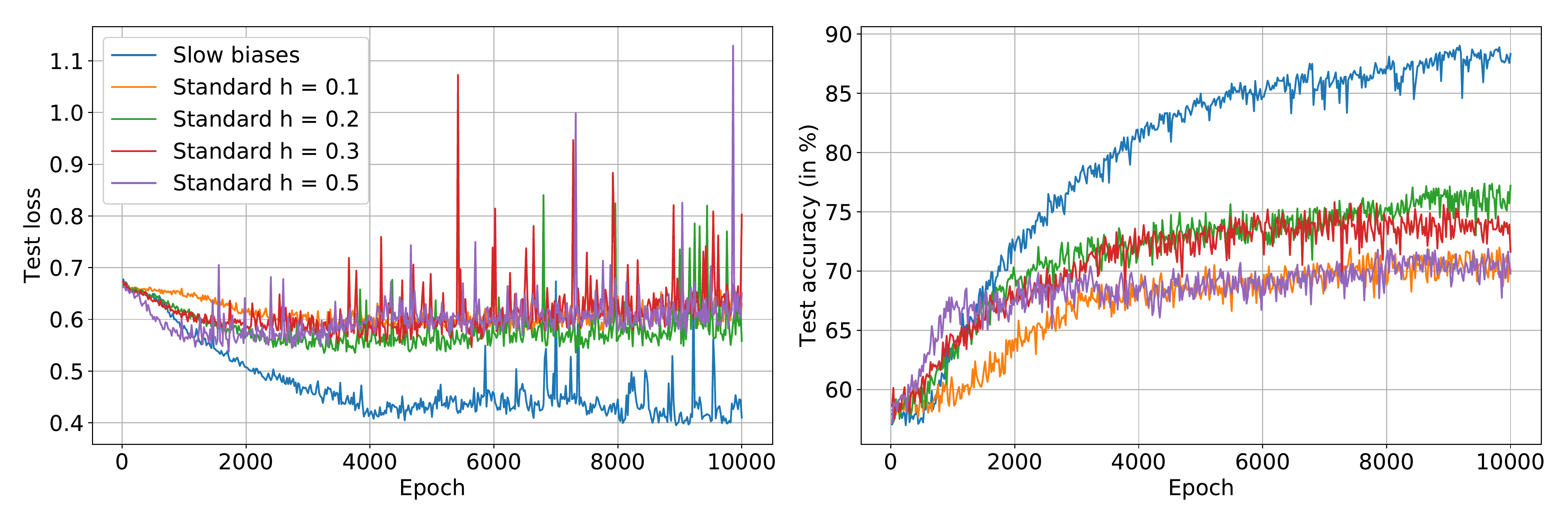}
    \caption{We use both standard SGD and our multirate approach as defined in Algorithm \ref{multirateSGDwithmom} to train a single hidden layer perceptron (SHLP) on a 4-turn spiral problem (adapted from \cite{LMV}) with 5\% subsampling. We set the biases of the neural network to be $\theta_S$ and the weights to be $\theta_F$ and set $k = 5$, $h=1$ in Algorithm \ref{multirateSGDwithmom}. Results are averaged over 5 runs.}
    \label{biasslow_spiral}
\end{figure}

\begin{figure}[h]
    \centering
    \includegraphics[width=0.9\linewidth]{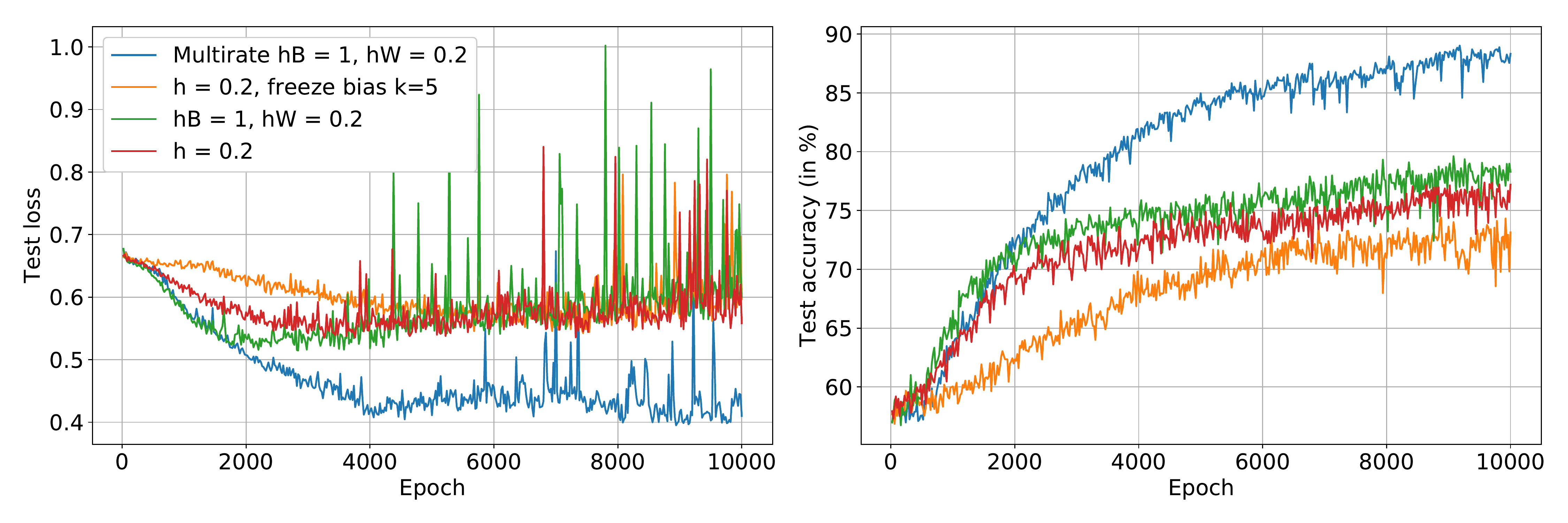}
    \caption{Same setting as in Figure A\ref{biasslow_spiral}, with a SHLP being trained on spiral data using SGD with $h = 0.2$ (red). We show that putting the biases on a different time scale (green) or freezing the biases for $k = 5$ steps and then updating them with the same stepsize as for the fast (weight) parameters (orange) both do not lead to the same performance improvement as our multirate technique (blue). }
    \label{fig:comparebiasslow_spiral}
\end{figure}
\begin{figure}[h]
    \centering
    \includegraphics[width=0.9\linewidth]{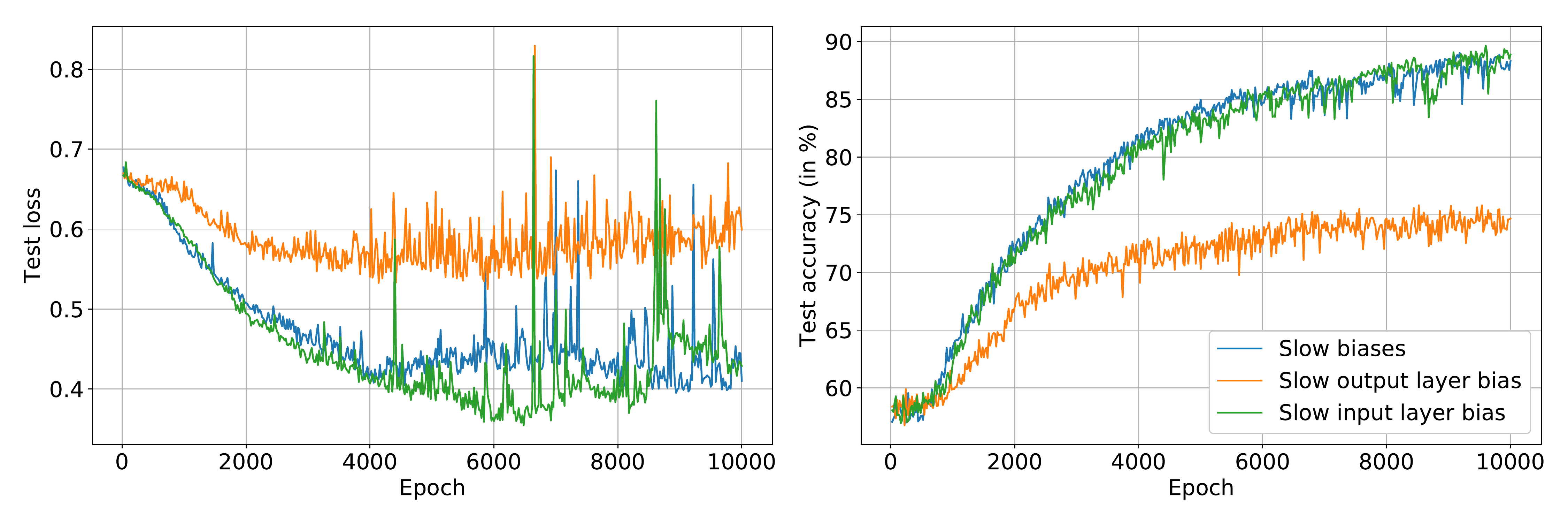}
    \caption{Same setting as in Figure A\ref{biasslow_spiral}, but here we study the effect of only putting the input biases (green) or only the output biases (orange) on the slow time scale. Clearly, using slow input biases (green) appears to be key to the enhanced generalization performance of the multirate approach (blue) in this setting.}
    \label{fig:whichbias}
\end{figure}

\begin{figure}[h]
    \centering
    \includegraphics[width=0.7\linewidth]{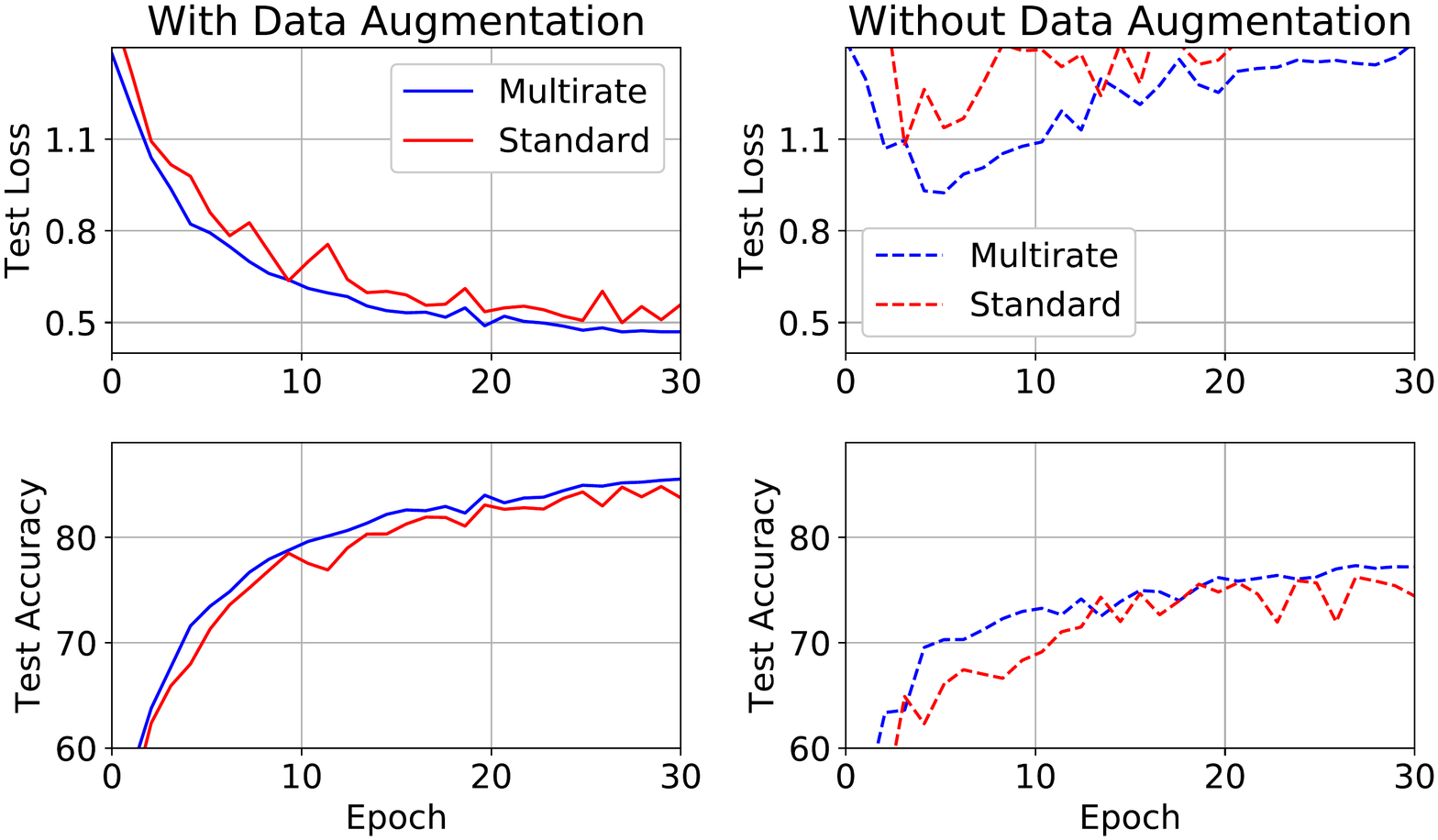}
    \vspace*{-8mm}
    \caption{We use both standard SGD (with $h = 0.1$) and our multirate approach as defined in Algorithm \ref{multirateSGDwithmom} to train a ResNet-34 architecture on CIFAR-10 data. We set the biases of the fully connected layer of the neural network to be $\theta_S$ and the weights to be $\theta_F$ and set $k = 10$, $h=1$ in Algorithm \ref{multirateSGDwithmom} and use batchsize = 128. Results are averaged over 10 runs.}
    \label{biasslow_resnet}
\end{figure}

\newpage
\section{Further Ablation Studies}\label{appx:ablation}

The effect of $k$ is studied in Table A\ref{tab:distilbert_varyk} for fine-tuning a pre-trained DistilBERT on SST-2 data using Algorithm 1. We find that although smaller values of $k$ can improve the generalization performance, the training time gets increased. This trade-off needs to be taken into account when choosing $k$. Apart from this, recall that lower values of $k$ also lower the stepsize used for the slow parameters, which can affect performance. Therefore, it may be beneficial to consider uncoupled learning rates, see discussion below and in Section \ref{subsec:multirate}. As a rule of thumb, we found that setting $k = 5$ often gives enough speed-up, without significantly affecting the accuracy.

We also provide ablation studies for the value of $k$ for the multirate approach for neural network regularization used in Section \ref{sec:randomsubgroups} for training from scratch a small MLP on MNIST data (Table A\ref{tab:mnist_varyk}, left) and a transformer on the Penn Treebank dataset (Table A\ref{tab:penntreebank_varyk}, left). In this setting the aim is not computational speed-up, but enhanced generalization performance. We again find that setting $k = 5$ gives optimal performance. In Table A\ref{tab:mnist_varyk} (right) and Table A\ref{tab:penntreebank_varyk} (right) we also show that using an uncoupled $h_S$ could lead to further performance enhancements, but this does introduce an additional hyperparameter to tune.

\begin{table}[h]
    \centering
        \caption{\textbf{Effect of $k$.} A pre-trained DistilBERT being trained on SST-2 data using our multirate approach for different values of $k$ (same setting as in Figure 7), where the fast parameters are set to be the fully connected (fc) layer + optionally the final attention block (layer 5). We set $h/k$ = 1e-4 and $\mu = 0.9$ in Algorithm 1 and use a batchsize of 16. 
    Results are presented over 10 runs. }
    \label{tab:distilbert_varyk}
    \begin{tabular}{c||c|c|c|c|c}
            Fast params  & $k$ & Mean test acc & Min test acc & Max test acc  & Avg Time (s) \\ \hline
Layer fc & $k = 3$ & 89.26\% & 88.69\% & 90.01\% & 245 \\
& $k = 5$  & \textbf{89.43\%} & 87.92\% & 90.28\% & 198 \\
& $k = 10$ & 88.53\% & 84.79\% & 89.79\% & 180 \\ \hline
Layer 5 + fc& $k = 3$ & 88.91\% & 87.42\% & 90.06\% & 264 \\
&  $k = 5$   & \textbf{89.70\%} & 89.35\% & 90.23\% & 224 \\
& $k = 10$ & 88.65\% & 87.97\% & 89.73\% & 207\\
    \end{tabular} 
\end{table}

\begin{table}[h]
    \centering
        \caption{ A single hidden layer perceptron trained on MNIST data using our multirate approach for neural network regularization (Section \ref{sec:randomsubgroups}) with $h_F = 0.1$. Left: different values of $k$ with coupled $h_S = k h_F$. Right: $k = 5$ and uncoupled $h_S$. Weights in the input and hidden layer are de-activated with a probability of 0.8 and 0.5, respectively (same setting as in Figure \ref{fig:randomsubset}).
    Results are presented over 10 runs. }
    \label{tab:mnist_varyk}
            \begin{tabular}{c||c|c|c}
            $k$ & Mean test acc & Min test acc & Max test acc \\ \hline
$k = 3$ & 76.16\% & 61.31\% & 89.41\% \\ 
$k = 5$ & \textbf{98.30}\% & 98.17\% & 98.44\% \\
$k = 10$ & 98.22\% & 98.11\% & 98.29\%
    \end{tabular} \quad
    \begin{tabular}{c||c|c|c}
            $h_S$ & Mean test acc & Min test acc & Max test acc \\ \hline
$h_S = 0.2$ & 98.26\% & 98.14\% & 98.37\% \\ 
$h_S = 0.5$ & 98.30\% & 98.17\% & 98.44\% \\
$h_S = 0.8$ & \textbf{98.31}\% & 98.21\% & 98.44\% \\
$h_S = 1$\ \  & 98.29\% & 98.19\% & 98.44\% \\
    \end{tabular} 
\end{table}

\begin{table}[h]
    \centering
        \caption{
        A transformer trained on Penn Treebank data 
        (same setting as in Figure \ref{fig:randomsubsetpenntreebank}) using our multirate approach for neural network regularization (Section \ref{sec:randomsubgroups}) with $h_F = 0.1$. Left: different values of $k$ with coupled $h_S = k h_F$. Right: $k = 5$ and uncoupled $h_S$. 
    Results are presented over 10 runs. }
    \label{tab:penntreebank_varyk}
    \begin{tabular}{c|c}
             $k$ & Minimum Validation Loss 
             \\ \hline
$k = 3$ & 4.870 $\pm$0.297 \\
$k = 5$  & \textbf{4.825} $\pm$0.302  \\
$k = 10$ & 4.838 $\pm$0.303\\ 
    \end{tabular} \qquad
          \begin{tabular}{c|c}
             $h_S$ & Minimum Validation Loss 
             \\ \hline
$h_S = 0.2$ & 4.824 $\pm$0.304\\
$h_S = 0.3$  & \textbf{4.823} $\pm$0.304  \\
$h_S = 0.5$  & 4.825 $\pm$0.302  \\
$h_S = 0.8$  & 4.831 $\pm$0.298  \\
    \end{tabular} 
\end{table}

\end{document}